\documentclass[10pt]{article}
\usepackage{appendix}
\usepackage{amsmath,graphicx,hyperref}
\usepackage{amsmath, amssymb, epsfig, graphicx, bm}
\usepackage{cases}

\usepackage[left=1in,top=1in,right=1in,bottom=1in,letterpaper]{geometry}

\usepackage{float}
\usepackage{subfigure}
\usepackage[utf8]{inputenc} 
\usepackage[T1]{fontenc} 
\usepackage{graphicx}
\usepackage{color}
\usepackage{xcolor} 
\usepackage[table]{xcolor}
\usepackage{makecell}
\usepackage{epstopdf}
\usepackage{amsthm}
\usepackage{booktabs}
\usepackage{pgfplots}
\usepackage{booktabs} 
\usepackage{multirow}
\usepackage{underscore}
\usepackage{amsfonts}
\usepackage{pgfplots}
\usepackage{pgfplotstable}
\usepackage{authblk}
\pgfplotsset{compat=1.17} 
\usepackage{cite, url}
\usepackage[all,cmtip]{xy}
\usepackage{color,hyperref}
\usepackage{algorithm}
\usepackage{algpseudocode}
\algtext*{EndFor}
\algtext*{EndIf}
\algtext*{EndFunction}
\usepackage{flushend}
\usepackage{graphicx}
\usepackage{textcomp}
\usepackage{xcolor}
\usepackage{array}
\usepackage{pifont}     
\usepackage{multirow}

\newtheorem{lemma}{Lemma}

\newtheorem{remark}{Remark}
\newtheorem{theorem}{Theorem}

\newcommand{\NucGD}{{\texttt{\textbf{NucGD}}}}



\title{\textbf{Towards The Implicit Bias on Multiclass Separable Data \\under Norm Constraints}}
\author{
  Shengping Xie\thanks{\texttt{2300010702@stu.pku.edu.cn}}, Zekun Wu\thanks{\texttt{2300010701@stu.pku.edu.cn}}, Quan Chen\thanks{\texttt{2300010617@stu.pku.edu.cn}}, Kaixu Tang\thanks{\texttt{2300012401@stu.pku.edu.cn}} \\ 
  \textit{School of Mathematical Sciences, Peking University} \\
}
\date{}

\begin{document}

\maketitle

\vspace{-3em}
\begin{abstract}
Implicit bias induced by gradient-based algorithms is essential to the generalization of overparameterized models, yet its mechanisms can be subtle. This work leverages the \emph{Normalized Steepest Descent} (NSD) framework to investigate how optimization geometry shapes solutions on multiclass separable data.
We introduce \NucGD\, a geometry-aware optimizer designed to enforce low rank structures through nuclear norm constraints. Beyond the algorithm itself, we connect \NucGD\ with emerging low-rank projection methods, providing a unified perspective. To enable scalable training, we derive an efficient \emph{SVD-free} update rule via asynchronous power iteration.
Furthermore, we empirically dissect the impact of stochastic optimization dynamics, characterizing how varying levels of gradient noise induced by mini-batch sampling and momentum modulate the convergence toward the expected maximum margin solutions.
Our code is accessible at:\\ \url{https://github.com/Tsokarsic/observing-the-implicit-bias-on-multiclass-seperable-data}.
\end{abstract}

\section{Introduction}
\vspace{-1em}
A key mystery in deep learning is how overparameterized models, such as neural networks, achieve strong generalization. Despite having the capacity to overfit, these models spontaneously learn simple representations through gradient-based training. This behavior suggests the existence of an "\textbf{implicit bias}", where the optimization algorithm itself favors low-complexity models. 

The phenomenon of implicit bias is currently best characterized within the context of unregularized Empirical Risk Minimization (ERM) on linearly separable data. Seminally, Soudry et al.  \cite{soudry2018implicit} demonstrated that for linear models trained via Gradient Descent (GD) with exponential loss, the weight vector asymptotically converges to the direction of the maximum $\ell_2$ margin SVM solution. This foundational result was recently extended to the multiclass setting by Ravi et al. \cite{ravi2024implicit}, confirming the universality of GD's preference for margin maximization.

However, modern large-scale model training relies heavily on adaptive gradient methods. Recent studies \cite{zhang2024implicit,xie2024implicit} have revealed that optimizers like Adam and AdamW exhibit an implicit bias associated with the $\ell_\infty$  norm, converging to the maximum $\ell_\infty$  margin solution. Furthermore, matrix-preconditioned methods typified by Muon \cite{liu2025muon} and SHAMPOO \cite{gupta2018shampoo} have been interpreted as approximations of steepest descent under spectral norm constraints \cite{riabinin2025gluon,lau2025polargrad}.

To integrate these diverse findings, Fan et al. \cite{fan2025implicit} proposed the \textbf{Normalized Steepest Descent} (NSD) framework, which models optimization updates as steepest descent steps constrained by a specific norm ball:
\begin{equation*}
\bm{\Delta_t} := \arg\max_{\|\bm{\Delta}\|_{\cdot}\leq \gamma} \langle \bm{\nabla_t}, \bm{\Delta}\rangle, 
\quad
\bm{x}_{t+1} = \bm{x}_{t} - \bm{\Delta_t}
\end{equation*}
Theoretical analysis within this framework establishes that on multiclass separable data, the convergence direction aligns with the maximum margin solution corresponding to the specific norm. Furthermore, replacing the gradient $\bm{\nabla_t}$ with the momentum $\bm{M}_t$ does not alter the results. 

This paradigm of implicit bias analysis exhibits strong generality. Specifically, when performing steepest descent with the spectral norm constraint and replacing $\bm {\nabla_t}$ with momentum $\bm{M}_t$, the corresponding optimizer
is Muon. When the $\ell_\infty$   norm is adopted as the constraint, it corresponds to SignGD, while Adam can be viewed as a smoothed version of SignGD~\cite{lau2025polargrad}. This connection thus unifies the previously fragmented
theoretical frameworks into a coherent narrative.
\paragraph*{Contributions.} Building on the previous research, our contributions contains two aspects:
\begin{itemize}
\item\textbf{Proposed Optimizer with Low-Rank Implicit Bias}

We propose \NucGD\, a novel optimizer derived from nuclear norm constraints, and investigate its implicit bias towards the maximum nuclear norm margin solution, which naturally promotes parsimonious low-rank representations. Crucially, we establish a novel link between \NucGD\ and popular low-rank gradient projection methods (e.g., Galore~\cite{zhao2024galore}), offering a unified theoretical justification for their efficacy. To make \NucGD\ practical, we further derive an \emph{SVD-free} implementation using asynchronous power iteration and discuss its asymptotic property.

\item \textbf{Analysis of Implicit Bias under Stochastic Optimization Dynamics.}

We implement NSD algorithms for multiclass linear models under various geometric constraints($\ell_2$, $\ell_\infty$, spectral, nuclear), and dissect the effect of stochastic optimization dynamics, specifically focusing on gradient noise induced by mini-batch sampling and trajectory smoothing via momentum.Through these experiments, we empirically observe and verify the convergence behavior of implicit bias in multiclass classification scenarios.

\end{itemize}
\section{Preliminaries}
In this section, we mainly follow the setting used in \cite{fan2025implicit}.\\
\noindent \textbf{Multiclass linear model.}Consider multiclass classification problem with training data $\boldsymbol{x}_1, \dots, \boldsymbol{x}_n$ and labels $y_1, \dots, y_n$. Each data point $\boldsymbol{x}_i \in \mathbb{R}^d$ is a $d$-dimensional embedding space (denote data matrix $\boldsymbol{X} = [\boldsymbol{x}_1, \dots, \boldsymbol{x}_n]^\top \in \mathbb{R}^{n \times d}$), and each label $y_i \in [k]$ represents one of $k$ classes. We assume each class contains at least one data point. The classifier $f_{\boldsymbol{W}}: \mathbb{R}^d \to \mathbb{R}^k$ is a linear model with weight matrix $\boldsymbol{W}=[\boldsymbol{w}_1,\boldsymbol{w}_2,\cdots ,\boldsymbol{w}_k]^\top\in \mathbb{R}^{k \times d}$. Each $w_i \in \mathbb{R}^{d}$ denotes a linear model $f_i(\boldsymbol{x})=\boldsymbol{w}_i^\top \boldsymbol{x}$ of class $i$ .

We train the model using empirical cross entropy loss: 
\[
\mathcal{L}(\boldsymbol{W}) := -\frac{1}{n} \sum_{i \in [n]} \log \left( \mathcal{S}_{y_i}(\boldsymbol{W} \boldsymbol{x}_i) \right). \tag{1}\label{lossfunction}
\]
Where the softmax map is defined as: $\mathcal{S}_c(\boldsymbol{a}) = \left[ \frac{\exp(\boldsymbol{a}[c])}{\sum_{c'=1}^k \exp(\boldsymbol{a}[c'])} \right]_{c=1}^k$. \\
\\
\noindent \textbf{Separable Data and Max Margin Solutions}
We define the SVM margin for a well-trained classifier to evaluate its robustness and generalization ability as follow:
$$m(\boldsymbol{W})=\underset{i \in [n],y \neq y_i}{\text{min}}(\boldsymbol{W}\boldsymbol{x}_i[y_i]-\boldsymbol{W}\boldsymbol{x}_i[y])=\underset{i \in [n],y \neq y_i}{\text{min}}(\boldsymbol{w}_{y_i}^\top \boldsymbol{x}_i-\boldsymbol{w}_{y}^\top \boldsymbol{x}_i)$$
Then $m(\boldsymbol{W}) \ge 0$ is equivalent to a correct classification of all data samples. To analyze the implicit regularization of optimization algorithms, we need to assume the data are \textbf{linearly separable}, so that we can further analyze its behavior after perfectly fit the training data:
$$\exists\boldsymbol{W} \ \text{s.t.} \quad m(\boldsymbol{W}) >0$$
As the model is linear, if we let $\boldsymbol{W'}=c\boldsymbol{W'}$ and $c \to +\infty$, we can reach $m(\boldsymbol{W}) \to +\infty$, however, we can evaluate the generalization properties for the weights with its relative margin computed by dividing its norm:\begin{align*}m_\textbf{.}(\boldsymbol{W})=\frac{m(\boldsymbol{W})}{\|\boldsymbol{W}\|_\textbf{.}} \tag{2}\label{relativemargin}\end{align*}
Where $\|\cdot\|_\textbf{.}$ can be replaced by any matrix norm. 

As $\|\cdot\|_\textbf{.}$ is linear, this metric can reflect the the separability to the data of the direction of $\boldsymbol{W}$. So We can intuitively reason that the higher $m_\textbf{.}(\boldsymbol{W})$ is, the better generalization ability the model contains. However, since different norms capture distinct properties of the weight matrix, the optimal directions under different norms may not align. We define the \textbf{max SVM margin solution} under norm $\|\cdot\|_\textbf{.}$ as follow:
\begin{align*}\boldsymbol{W^\textbf{.}}=\underset{\|\boldsymbol{W}\|_\textbf{.}=1}{\text{argmax}} \ m(\boldsymbol{W}) \tag{3} \label{maxmargin}\end{align*}
Here, we mainly focus on common norms such as Frobenius norm $\|\boldsymbol{W}\|_F$, $\ell_\infty$ norm $\|\boldsymbol{W}\|_{max}$, spectral norm $\|\boldsymbol{W}\|_2$ and nuclear norm $\|\boldsymbol{W}\|_{*}$\\
\noindent \textbf{Normalized Steepest Descent.} We consider a family of optimization algorithms for an objective function $f$ under a specific norm $\|\cdot\|_\textbf{.}$ with an optional first-order momentum. At each training step, the algorithm first computes the gradient and accumulates it into the momentum state. However, instead of directly applying the momentum to the weights, the algorithm updates the parameters along the steepest descent direction of the momentum within a norm-constrained ball.
\begin{align*}
\bm{G}_t&=\bm{\nabla} f(\bm{x}_t),\\
\bm{M}_t&=\mu \bm{M}_{t-1}+(1-\mu)\bm{G}_t,\\
{\bm{\Delta}}_t &= \underset{\|{\bm{\Delta}}\|_\textbf{.} \leq \gamma}{\text{argmax}} \langle \bm{M}_t, {\bm{\Delta}} \rangle, \\\bm{x}_{t+1}&=\bm{x}_t-{\bm{\Delta}}_t,
\label{alg-nsd} \tag{4}
\end{align*}
Where $\gamma>0$ and $\mu \in [0, 1)$ denotes the step size and the momentum weight, respectively. When $\mu=0$ and $f$ is linear, it is equivalent to do a steepest descent in the respective norm-constraint ball.\\
The above framework contains a wide range of existing popular training algorithms for deep learning. According to ~\cite{fan2025implicit}, the detailed information of each algorithm and their corresponding norm can be shown in Table ~\ref{tab:nsd-alg}.

\begin{table}[h]
\centering

\label{tab:nsd-alg}
\vspace{1em}
\begin{tabular}{|l|l|l|l|}
\hline
Method & Norm Constraint & Update $\boldsymbol{\Delta}$ & Reference  \\
\hline\hline
NGD & \multirow{2}{*}{Unit $\|\cdot\|_2$-ball} & $\frac{\boldsymbol{G}}{\|\boldsymbol{G}\|_2}$ & Hazan et al. \cite{hazan2015beyond}  \\
\cline{1-1} \cline{3-4}
NMD-GD &  & $\frac{\boldsymbol{M}}{\|\boldsymbol{M}\|_2}$ & Cutkosky and Mehta \cite{cutkosky2020momentum}  \\
\hline
SignGD & \multirow{2}{*}{Unit $\|\cdot\|_{\text{max}}$-ball} & $\text{sign}(\boldsymbol{G})$ & Bernstein et al. \cite{bernstein2018signsgd} \\
\cline{1-1} \cline{3-4}
Signum &  & $\text{sign}(\boldsymbol{M})$ & Bernstein et al. \cite{bernstein2018signsgd} \\
\hline
{Spectral-GD} & \multirow{2}{*}{Unit $\||\cdot|\|_\infty$-ball} & $UV^T$ & Bernstein and Newhouse \cite{bernstein2024old} \\
\cline{1-1} \cline{3-4}
{Muon}  &  & $\tilde{U}\tilde{V}^T$ & Liu et al. \cite{liu2025muon} \\
\hline
\end{tabular}
\caption{Summary of existing algorithms under the framework of NSD under respective norm constraints. Here $U,V,\tilde{U},\tilde{V}$ are computed by the SVD decomposition of the gradient or momentum, and the algorithms listed above each group correspond to the momentum-free method.}
\label{norm}
\end{table}

\noindent{\textbf{Theoretical Analysis of Implicit Bias.}}
Fan et al. \cite{fan2025implicit} has presented a theoretical analysis on implicit bias for Normalized Steepest Algorithms \eqref{alg-nsd} 
\begin{theorem}
For multiclass linear model defined in equation~\eqref{lossfunction} with separable data, under basic assumptions, when the step size is taken of order $\sqrt{1/t}$, the relative margin of the weight matrix equation ~\eqref{relativemargin} follow the iterate of NSD will converge to the maximum relative margin:
$$m_\textbf{.}(\boldsymbol{W^\textbf{.}}) - m_\textbf{.}(\boldsymbol{W}_T)\leq\mathcal{O}\left(\frac{\log T}{T^{1/2}}\right)$$
\label{convergence}
\end{theorem}
\vspace{-2em}
We can also imply the direction of the weight matrix also align well with the max margin solution in equation~\eqref{maxmargin} under the corresponding norm.
$$\lim_{t \to \infty}\frac{\boldsymbol{W}_t}{\|\boldsymbol{W}_t\|_\textbf{.}} \approx  \boldsymbol{W^\textbf{.}}$$

\section{\NucGD: Normalized Steepest Descent Under Nuclear Norm}
Theorem~\ref{convergence} motivates a promising strategy, indicating that NSD-type algorithms have the potential to induce solutions with \emph{favorable} structural properties. By explicitly solving norm-constrained optimization problems, we can derive algorithms that steer the model towards the desired direction.

Beyond the norms considered previously,
The \textit{nuclear norm} $\|A\|_{*}$ of a matrix $A \in \mathbb{R}^{m \times n}$ is the sum of its singular values: $\|A\|_{*} = \sum_i \sigma_i$. As $\sigma_i \geq 0$ by definition, $\|A\|_{*}$ is also the $L^1$ norm of its vector of singular values and it acts as the tightest convex surrogate of the matrix rank. 

Just as $L^1$ regularization steers learning problems towards solutions with many zero entries, nuclear norm regularization steers \textit{matrix} learning problems towards \textit{low-rank} solutions with many zero singular values\cite{scarvelis2024nuclear}, which encourages low rank models. Low-rank models are known to capture the intrinsic global structure and latent relationships in data more faithfully, which in turn often leads to improved generalization performance in tasks . For this reason, nuclear norm regularization has seen widespread use in deep learning. However, explicitly adding nuclear norm as regularization term is computational costly, as computing the nuclear norm of a matrix requires full SVD decomposition which is impractical in large scale training. 

Meanwhile, relying on Theorem ~\ref{convergence} and above analysis, NSD algorithms under nuclear norm potentially contains an implicit bias toward low rank models as it converge to the max margin solutions under nuclear norm for a linear model, therefore it is expected to play the similar role as explicit low rank regularization. To further validate this observation,we visualized the detailed information of the max margin solutions under nuclear norm and the previous norms appeared in Table \ref{norm} in Figure ~\ref{fig:spectral} ~\ref{fig:heatmap}. 

\begin{figure}[h]
    \centering
    \subfigure[correlation between different matrix]{
        \includegraphics[width=0.4\textwidth]{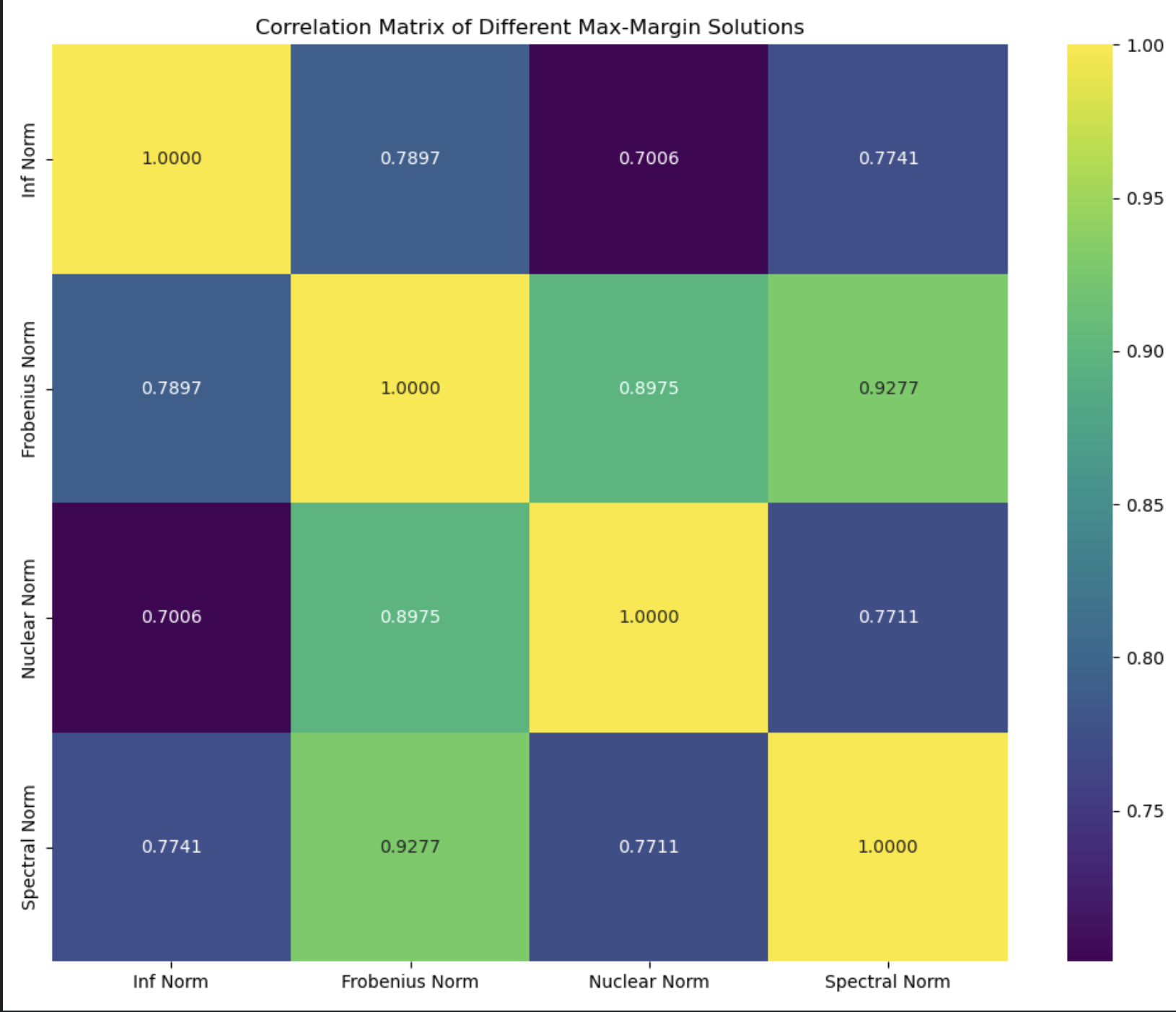}
    }
    \subfigure[spectrum of each matrix]{
        \includegraphics[width=0.35\textwidth]{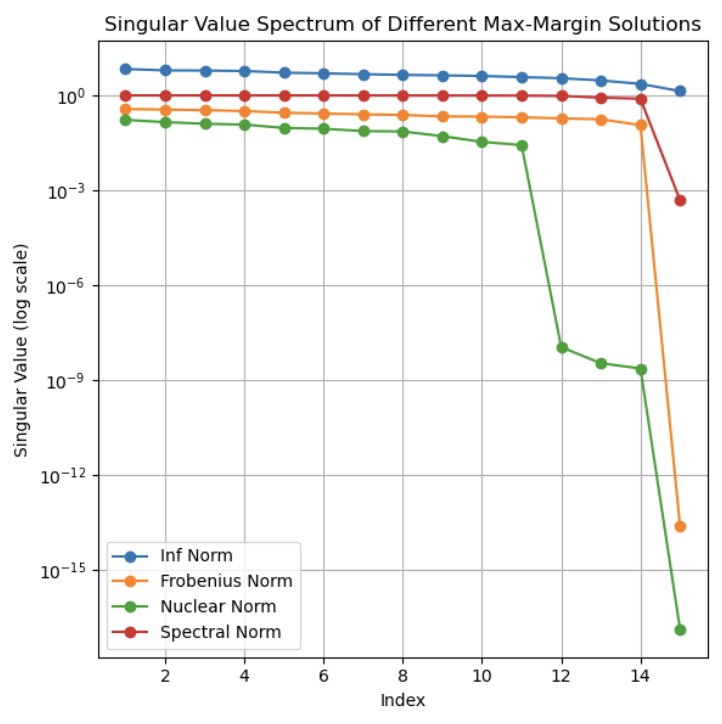}
    }
    \caption{Correlations and spectrum for max margin solutions under different norm}
    \label{fig:spectral}
\end{figure}

\begin{figure}
    \centering
    \includegraphics[width=0.95\linewidth]{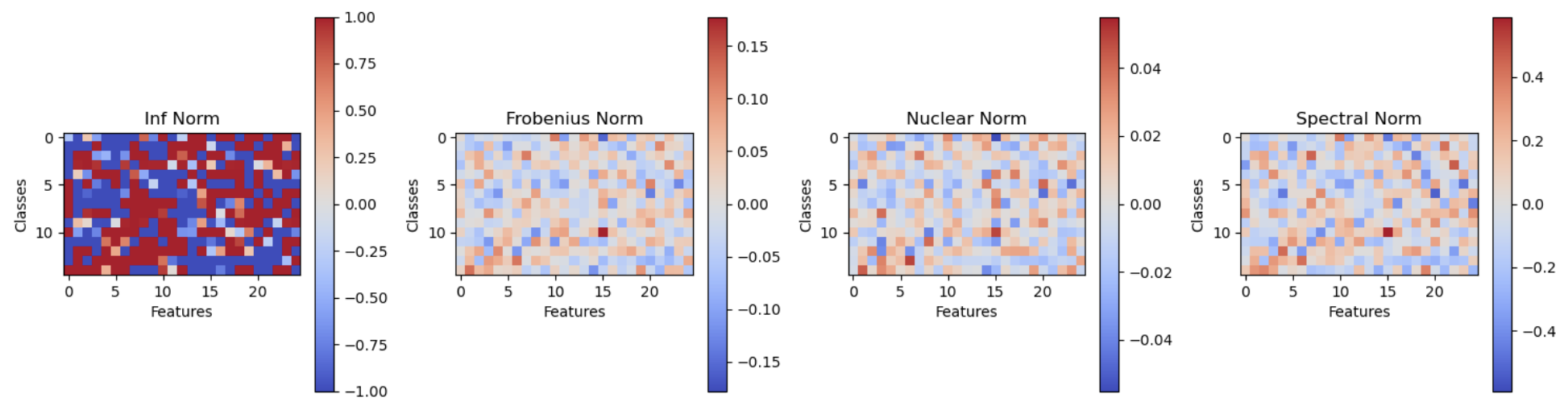}
    \caption{The Weight Heatmap of max margin solutions under different norm}
    \label{fig:heatmap}
\end{figure}
From Figure ~\ref{fig:spectral}, though the direction of different solutions are relatively similar, the nuclear norm max-margin solution demonstrated a much better low-rankness compared to other solutions with 4 of the spectrum close to 0. Previous verification reflects the deriving of NSD under nuclear norm is of value.

To formulate the \NucGD\ Algorithm, Theorem ~\ref{nuctrm} provides us a theoretical approach to compute the steepest descent update direction.

\begin{theorem}
Consider a matrix $\boldsymbol{M} \in \mathbb{R}^{k \times d}$ with its Singular Value Decomposition (SVD) given by $\displaystyle \boldsymbol{M}=\sum_{i=1}^{\min\{k,d\}} \sigma_i \boldsymbol{u}_i \boldsymbol{v}_i^\top$. Here, $\sigma_1 \ge \sigma_2 \ge \cdots \ge 0$ are the singular values, and $\boldsymbol{u}_i \in \mathbb{R}^{k}$, $\boldsymbol{v}_i \in \mathbb{R}^{d}$ denote the corresponding orthogonal and unit left and right singular vectors, respectively. Then we have:
$$\gamma u_1v_1^\top \in \underset{\boldsymbol{\Delta}\in \mathbb{R}^{k \times d},\|{\boldsymbol{\Delta}}\|_{*} \leq \gamma}{\arg\max} \langle \boldsymbol{M}, {\boldsymbol{\Delta}} \rangle$$
\label{nuctrm}
\end{theorem}
\begin{proof}
By duality of nuclear norm and spectral norm:
\[
\langle \boldsymbol{M}, \boldsymbol{\Delta} \rangle \leq \||\boldsymbol{M}|\|_\infty \cdot \|\boldsymbol{\Delta}\|_* = \sigma_1 \cdot \|\boldsymbol{\Delta}\|_* \leq \sigma_1 \gamma.
\]
For $\boldsymbol{\Delta}^* = \gamma {u}_1{v}_1^\top$, we have $\|\boldsymbol{\Delta}^*\|_* = \gamma \cdot \|{u}_1 {v}_1^\top\|_* = \gamma$ and $\langle \boldsymbol{M}, \boldsymbol{\Delta}^* \rangle = \gamma \text{tr}(\boldsymbol{M}^\top {u}_1 {v}_1^\top) = \gamma \sigma_1 \|u_1\|^2\|v_1\|^2 =\gamma \sigma_1$, so we finished the proof. Furthermore, if $\sigma_1 > \sigma_2$, $\gamma u_1 v_1^{T}$ is the unique solution.
\end{proof}

By Theorem ~\ref{nuctrm}, we can introduce the update rule of \NucGD\ in Algorithm ~\ref{alg:nucgd}.

\begin{algorithm}[t!]
	\caption{\NucGD(Analytic view)}
	\label{alg:nucgd}
	\begin{algorithmic}[0]
	\State \textbf{Require:} stepsize $\gamma$; momentum parameter $0 \leq \mu < 1$.
    \State \textbf{Initialize:} initial weight $\bm{W}_0 \in \mathbb{R}^{k \times d}$,  momentum $\bm{M}_{-1}= 0_{k \times d}$.
    \For{$t = 0, \ldots, T - 1$}
		\State $\bm{G}_t \gets \bm{\nabla} f(\bm{W}_t)$
        \State $\bm{M}_t \gets \mu\bm{M}_{t-1}+(1-\mu)\bm{G}_t$
        \State $\bm{U}_t,\bm{\Sigma}_t,\bm{V}_t \gets \textbf{SVD}(\bm{M}_t)$
        \State $\bm{\Delta}_t \gets \gamma\bm{U}_t[:,1]\bm{V}_t[1,:]$
        \State $\bm{W}_{t+1} \gets \bm{W}_{t}-\bm{\Delta}_t$
    \EndFor
    \State\Return ${\{\bm{W}_t\}}_{t=0}^{T}$
	\end{algorithmic}
\end{algorithm}

\NucGD\ as \textbf{low rank projections}: though the analytic iteration of \NucGD\ requires full SVD decomposition, we can rewrite the computation of steepest descent direction by following observation:
\begin{remark}
For matrix $\boldsymbol{M} \in \mathbb{R}^{k \times d}$ with its SVD decomposition $\displaystyle \boldsymbol{M}=\sum_{i=1}^{\min\{k,d\}} \sigma_i u_iv_i^\top$, we have:
$$u_1v_1^\top=\frac{u_1u_1^\top\boldsymbol{M}}{\|u_{1}^\top\boldsymbol{M}\|_2}$$
\label{remark1}
\end{remark}

\begin{proof}
Easy to compute $u_1 u_1^{\top}\bm{M}=\sigma_1u_1 u_1^\top u_1 v_1^{\top}=\sigma_1 u_1^\top v_1$, $\|u_1^\top \bm{M}\|=\sigma_1 \|u_1 u_1^\top v_1^\top\| = \sigma_1 \|v_1^\top\|=\sigma_1$. 
\end{proof}
By remark ~\ref{remark1}, the implementation of \NucGD \ only rely on the computation of its principle singular vector, which is also the \emph{principal eigenvector} of matrix $\boldsymbol{N}=\boldsymbol{M}\boldsymbol{M}^\top$.To estimate the dominant eigenvector of a matrix, \textbf{power iteration} is a widely used method featuring high computational efficiency and fast convergence. By successive matrix-vector multiplications and vector normalization, it can yield highly accurate dominant eigenvectors with minimal computational cost. 

We rely on the strategies used in \cite{vogels2019powersgd}, as momentum in adjacent steps usually have a high degree of similarity, we can use the estimated eigenvector of the previous step as a initialization of power iteration at current step, and further reduce the number of power iteration steps at every training step. Algorithm ~\ref{alg:nucgd-real} presents an efficient implementation \NucGD\ by doing one power iteration steps per training step rely on the historical estimations asynchronously. Furthermore, we establish a convergence result, stating that in practical conditions, the power iteration direction will converge to the exact descent direction of \NucGD\ . The details are in Appendix~\ref{proof}. 
\begin{algorithm}[t!]
	\caption{\NucGD\ with Power Iteration}
	\label{alg:nucgd-real}
	\begin{algorithmic}[0]
	\State \textbf{Require:} stepsize $\gamma$; momentum parameter $0 \leq \mu < 1$.
    \State \textbf{Initialize:} initial weight $\bm{W}_0 \in \mathbb{R}^{k \times d}$,  momentum $\bm{M}_{-1}= \bm{O}_{k \times d}$. Initial projection vector $\bm{p}_{-1} \sim \mathcal{N}(0,I_k) $
    \For{$t = 0, \ldots, T - 1$}
		\State $\bm{G}_t \gets \bm{\nabla} f(\bm{W}_t)$
        \State $\bm{M}_t \gets \mu\bm{M}_{t-1}+(1-\mu)\bm{G}_t$
        \State $\bm{p}_{t}\gets {\bm{M}_t\bm{M}_t^\top\bm{p}_{t-1}} /{\|\bm{M}_t\bm{M}_t^\top\bm{p}_{t-1}\|_2}$ \Comment{Do single step power iteration based on previous $p_{t-1}$}
        \State $\bm{\Delta}_t \gets \gamma\bm{p}_{t}\bm{p}_{t}^\top\bm{M}_t/{\|\bm{p}_{t}^\top\bm{M}_t\|_2}$ \Comment{Compute update direction based on Remark ~\ref{remark1}}
        \State $\bm{W}_{t+1} \gets \bm{W}_{t}-\bm{\Delta}_t$
    \EndFor
    \State\Return ${\{\bm{W}_t\}}_{t=0}^{T}$
	\end{algorithmic}
\end{algorithm}

\section{Experiments}
We verify the theoretical predictions of the NSD framework by reproducing the baseline results of Fan et al.~\cite{fan2025implicit} (see Appendix ~\ref{reproduction} for details).In this section, we focus on the analysis to extending settings.

\paragraph{Experimental setup.}
We generate a synthetic multiclass dataset by first sampling $k$ class centers $\{c_y\}_{y=1}^k$ i.i.d. from a standard normal distribution, and then sampling $n_{\text{per class}}$ points per class from $\mathcal{N}(c_y,\sigma^2 I_d)$.
In all experiments, we use $k=15$ classes, feature dimension $d=25$, $n_{\text{per class}}=50$ (thus $n=750$), noise level $\sigma=0.1$, and random seed $12344$.
We verify linear separability by solving the entrywise $\ell_\infty$ norm multiclass SVM feasibility problem.
We use a $k$-class linear classifier $f_W(x)=Wx$ (no bias) trained with cross-entropy loss.
Four optimizers (SignGD, NGD, Muon, NucGD) are evaluated.
Weights are initialized with the zero initialization.The theretical maximum margin solution is solved by \texttt{cvxpy}.
\paragraph{Metrics.} To quantify the implicit bias and convergence quality, we mainly employ two key metrics to compare the learned weight $\boldsymbol{W}$ (produced by the optimizer at step $T$) against the theoretical maximum margin solution $\boldsymbol{W}^*$ defined in Equation ~\eqref{maxmargin}:

\begin{itemize}
    \item[(i)] \textbf{ Correlation: }We treat the weights as vectors in the parameter space and utilize the cosine similarity based on the Frobenius norm:
    \begin{equation*}
        \text{Corr}(\boldsymbol{W}, \boldsymbol{W}^*) = \frac{\langle \boldsymbol{W}, \boldsymbol{W}^* \rangle}{\|\boldsymbol{W}\|_F \|\boldsymbol{W}^*\|_F},
    \end{equation*}
    where $\langle \cdot, \cdot \rangle$ denotes the standard Euclidean inner product. 
    \item[(ii)] \textbf{Margin error:} The \emph{normalized maximum margin} according to  Equation~\eqref{relativemargin} is:
    \begin{equation*}
        m_{{\|\cdot\|}_{\cdot}}(\boldsymbol{W}) = \frac{m(\boldsymbol{W})}{\|\boldsymbol{W}\|_{\cdot}}, \quad \text{where } m(\boldsymbol{W}) = \min_{i} \left( \boldsymbol{W}_{y_i}^\top x_i - \max_{j \neq y_i} \boldsymbol{W}_{j}^\top x_i \right).
    \end{equation*}
    We can define the relative margin error corresponding to the specific norm $\|\cdot\|_{\cdot}$ as
    \begin{equation*}
    \epsilon_{{\|\cdot\|}_{\cdot}}(W)=\frac{\mid m_{{\|\cdot\|}_{\cdot}}(\boldsymbol{W}) - m_{{\|\cdot\|}_{\cdot}}(\boldsymbol{W}^{*})\mid}{m_{{\|\cdot\|}_{\cdot}}(\boldsymbol{W}^{*})}
    \end{equation*}
\end{itemize}
\subsection{Experiments of \NucGD}

We empirically investigate \NucGD, and compare it with other NSD variants induced by $\ell_2$, $\ell_\infty$, and spectral norms. The goal of these experiments is to verify (from optimization trajectories and solution structure) that the learned classifier follows the implicit-bias prediction: the algorithm tends to approach the
\emph{max-margin direction} associated with its underlying norm geometry, and the induced geometry leaves a visible footprint on the singular-value spectrum of the learned weight matrix.

\paragraph{Results (from Fig.~\ref{fig:nucgd_four}).}
\textbf{(a) Max-margin spectra.}
The singular-value spectra of max-margin solutions under different norms show qualitatively distinct structures:
the nuclear norm max-margin solution exhibits a pronounced spectral collapse in the tail,
with several trailing singular values dropping close to numerical zero, indicating a substantially smaller effective rank
than the alternatives.
In contrast, the max-margin solutions under $\ell_2/\ell_\infty$/spectral-type constraints keep a much flatter tail
(without an equally sharp drop), suggesting comparatively higher-rank structure.

\smallskip
\textbf{(b) Normalized-margin error along NucGD.}
The margin-error curves measured against different max-margin references reveal a clear ordering:
the nuclear norm error $\mathrm{Err}_{\ast}(t)$ decays the most and attains the smallest terminal gap,
while errors against the $\ell_2$, spectral, and $\ell_\infty$ max-margin directions level off at larger values.
This indicates that the trajectory produced by NucGD is closest (in the normalized-margin sense) to the nuclear norm
max-margin solution among the candidates considered.

\smallskip
\textbf{(c) Alignment (correlations) along NucGD.}
Consistent with the margin-error view, the correlation curves show that NucGD achieves the highest alignment with the
nuclear norm max-margin direction and maintains the largest terminal correlation.
Alignments to $\ell_2$, spectral, and $\ell_\infty$ references saturate strictly lower, indicating that NucGD is not merely
approaching a generic max-margin direction but is selectively biased towards the nuclear norm one.

\smallskip
\textbf{(d) Final NSD solution spectra across algorithms.}
Finally, comparing the singular-value spectra of the \emph{final} solutions produced by different NSD algorithms,
the NucGD solution displays the strongest suppression of small singular values (the sharpest decay in the tail),
which matches the low-rank signature suggested by the nuclear max-margin spectrum in (a).
Other NSD variants (e.g., spectral and $\ell_\infty$) retain noticeably flatter tails, reflecting different structural biases.

\paragraph{Mechanistic interpretation (norm geometry $\Rightarrow$ low-rank updates).}
The above patterns can be explained directly from the geometry of steepest descent under the nuclear norm.
For a gradient matrix $\bm{G}_t=\bm{\nabla} \mathcal{L}(W_t)$, the NSD direction under $\|\cdot\|_{\ast}$ solves( here we assume the maximum singular value of $\bm{G_t}$ is single):
\[
\Delta_t \in \arg\min_{\|\Delta\|_{\ast}\le 1}\ \langle \bm{G}_t,\Delta\rangle
\quad\Longleftrightarrow\quad
\Delta_t = -\,u_1 v_1^{\top},
\]
where $u_1,v_1$ are the top left/right singular vectors of $\bm{G}_t$ (Theorem~\ref{nuctrm}).
Thus, each NucGD step is driven predominantly by a \emph{rank-one} outer product aligned with the leading singular mode of the gradient.
Accumulating such updates naturally concentrates mass on a small number of dominant singular directions,
which provides a principled explanation for (i) the rapid emergence of spectral tail collapse and (ii) the preferential
alignment with the nuclear norm max-margin direction observed in (b)--(c).
Importantly, this mechanism is specific to the nuclear/spectral dual pair and does not hold for $\ell_2$ or entrywise $\ell_\infty$ geometries,
hence the distinct spectra and alignment behavior across different NSD algorithms.

\begin{figure}[H]
\centering
\begin{minipage}[t]{0.24\textwidth}
  \centering
  \includegraphics[width=\linewidth]{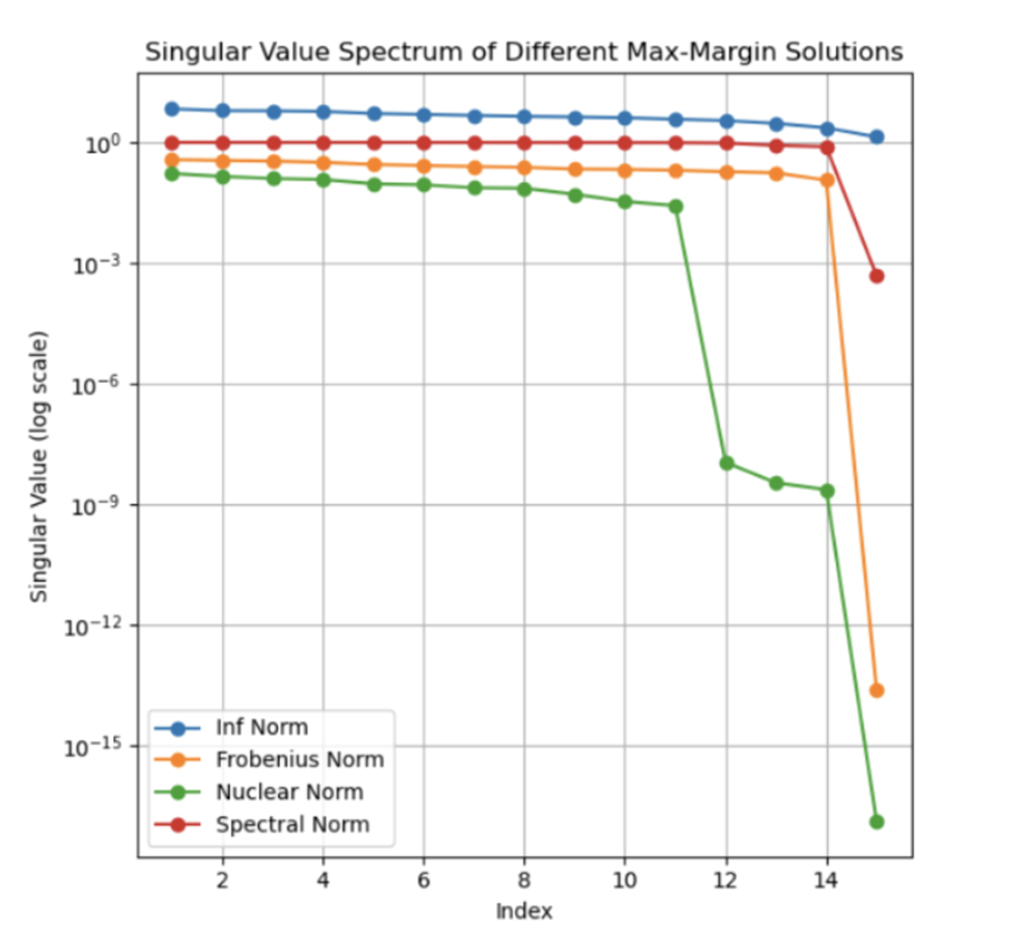}\\[-0.5ex]
  {\scriptsize (a) Max-margin spectra}
\end{minipage}\hfill
\begin{minipage}[t]{0.24\textwidth}
  \centering
  \includegraphics[width=\linewidth]{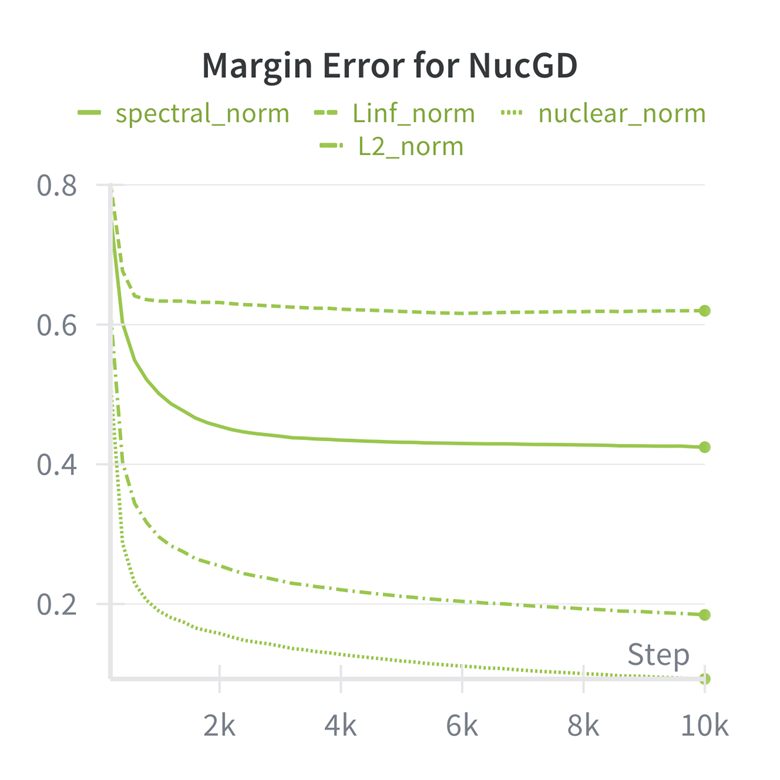}\\[-0.5ex]
  {\scriptsize (b) Margin error (NucGD)}
\end{minipage}\hfill
\begin{minipage}[t]{0.24\textwidth}
  \centering
  \includegraphics[width=\linewidth]{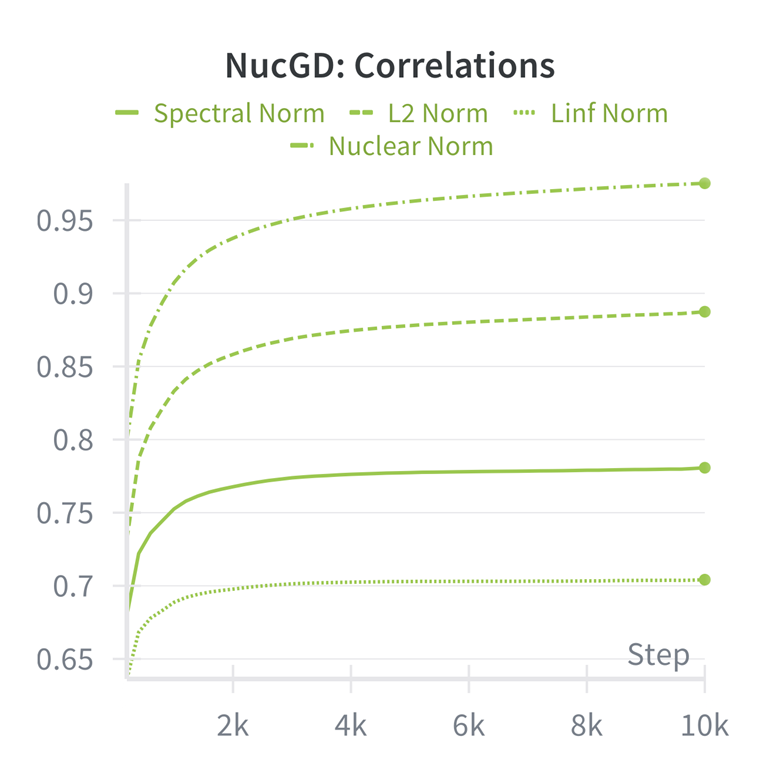}\\[-0.5ex]
  {\scriptsize (c) Correlations (NucGD)}
\end{minipage}\hfill
\begin{minipage}[b]{0.24\textwidth}
  \centering
  \includegraphics[width=\linewidth]{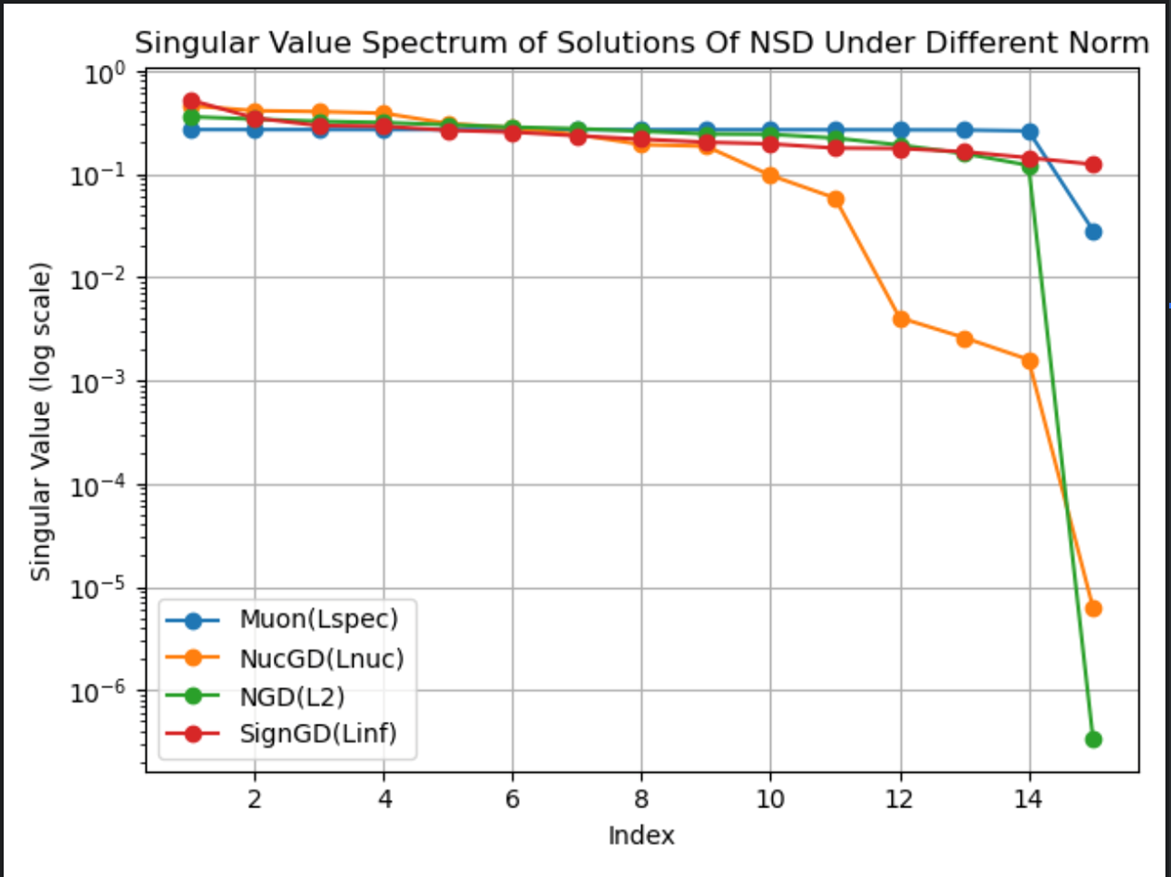}\\[-0.5ex]
  {\scriptsize (d) NSD solution spectra}
\end{minipage}

\caption{NucGD experiments: (a) singular value spectra of max-margin solutions under different norms;
(b) normalized-margin error along NucGD measured against different max-margin references;
(c) correlations between NucGD iterates and different max-margin directions;
(d) singular value spectra of final solutions produced by different NSD algorithms.}
\label{fig:nucgd_four}
\end{figure}

\subsection{Experiments under Stochastic Settings}

In this section, we study the impact of gradient noise on implicit bias. Specifically, we examine the effects of varying the mini-batch size $B$ and the momentum weight $\mu$.

For a fixed dataset, the mini-batch gradient
$
\hat {\bm{G}}_t=\frac{1}{B}\sum_{i\in\mathcal{B}_t}{\nabla} \ell_i(W_t)
$
serves as a stochastic estimator of the full gradient, whose variance typically decreases as $B$ increases (roughly $\propto 1/B$). 
Therefore, smaller batch sizes correspond to \emph{higher} gradient noise, while larger batch sizes lead to dynamics closer to full-batch training.

In ~\eqref{alg-nsd}, we introduce a parameter $\bm{M}_t$ to record the accumulated history information of first momentum, and mixture it with current gradient $\bm{G}_t$ with weight $\mu$. Under a stochastic setting, adding momentum can filter high frequency noise, which accelerate the gradient descent algorithm efficiently.Therefore, smaller momentum weights correspond to \emph{higher} gradient noise.

Two experiments are implemented:\newline
\textbf{(a)}We evaluate different batch sizes
$B\in\{1,5,50,250,750\}$, while keeping other hyperparameters fixed.\newline
\textbf{(b)}We evaluate different coefficients $\mu$ for both full-batch and mini-batch gradients. Note that the range of $\mu$ is adapted to each setting to ensure numerical stability.

\paragraph{Results of (a).}
Figures~\ref{fig:noise} visualize how the final alignment (correlation) and the normalized margin vary with the batch size.
Overall, the impact of noise is optimizer-dependent:\newline
(i) for \textbf{Muon}, the smallest batch size $B=1$ leads to unstable optimization (low accuracy and negative normalized margin), while moderate/larger batch sizes recover perfect classification and strong alignment;\newline
(ii) for \textbf{NGD}, larger batch sizes (less noise) consistently improve both the final normalized margin and correlation;\newline
(iii) for \textbf{NucGD}, smaller batch sizes yield \emph{higher} correlation, suggesting that a certain level of noise can promote alignment with the max-margin direction for this optimizer;\newline
(iv) for \textbf{SignGD}, the best correlation is attained at a small batch size ($B=5$), and then slightly decreases as $B$ increases.

\begin{figure}[H]
  \centering
  \subfigure[Margin error of Muon]{%
    \includegraphics[width=0.23\textwidth]{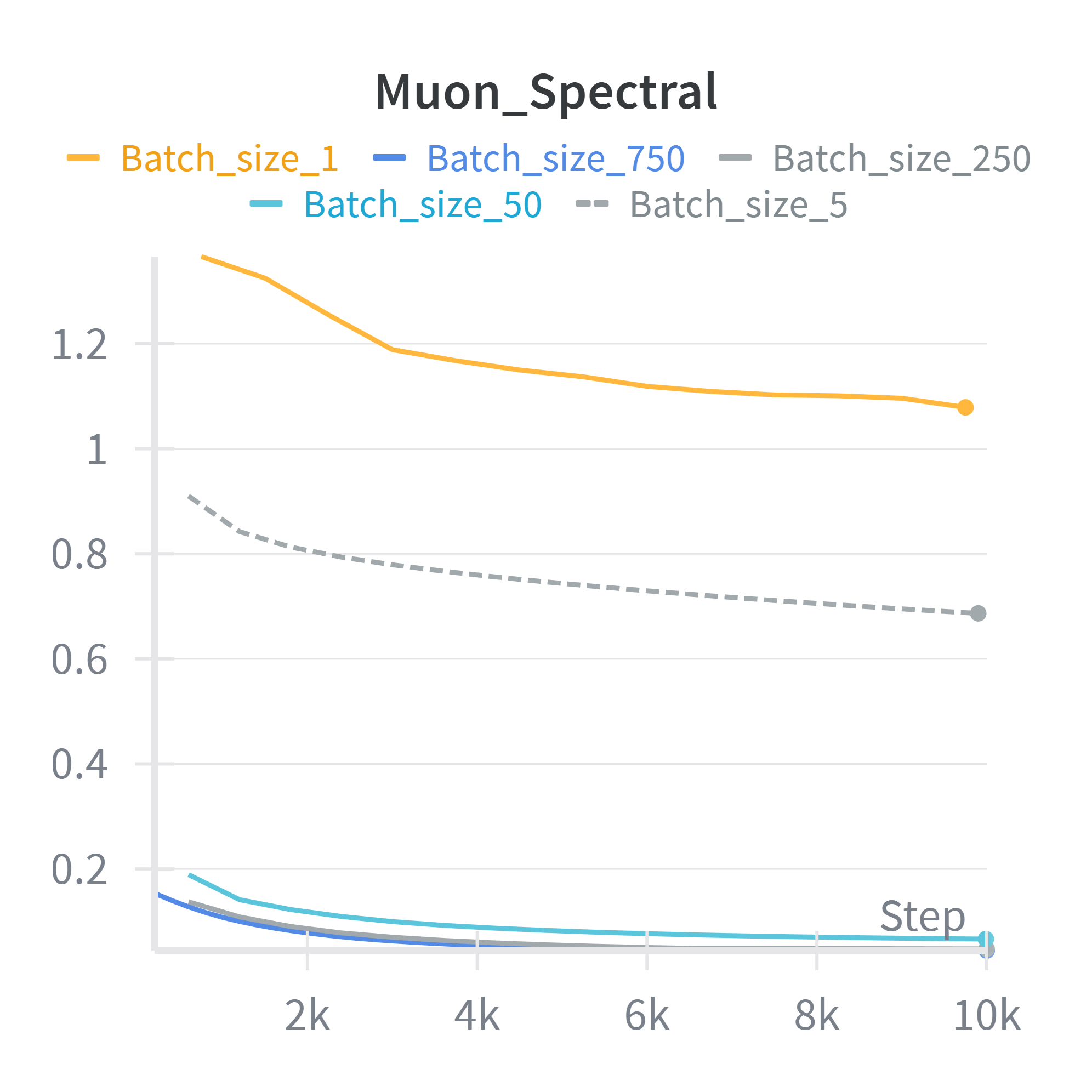}%
  }\hfill
  \subfigure[Margin error of NGD]{%
    \includegraphics[width=0.23\textwidth]{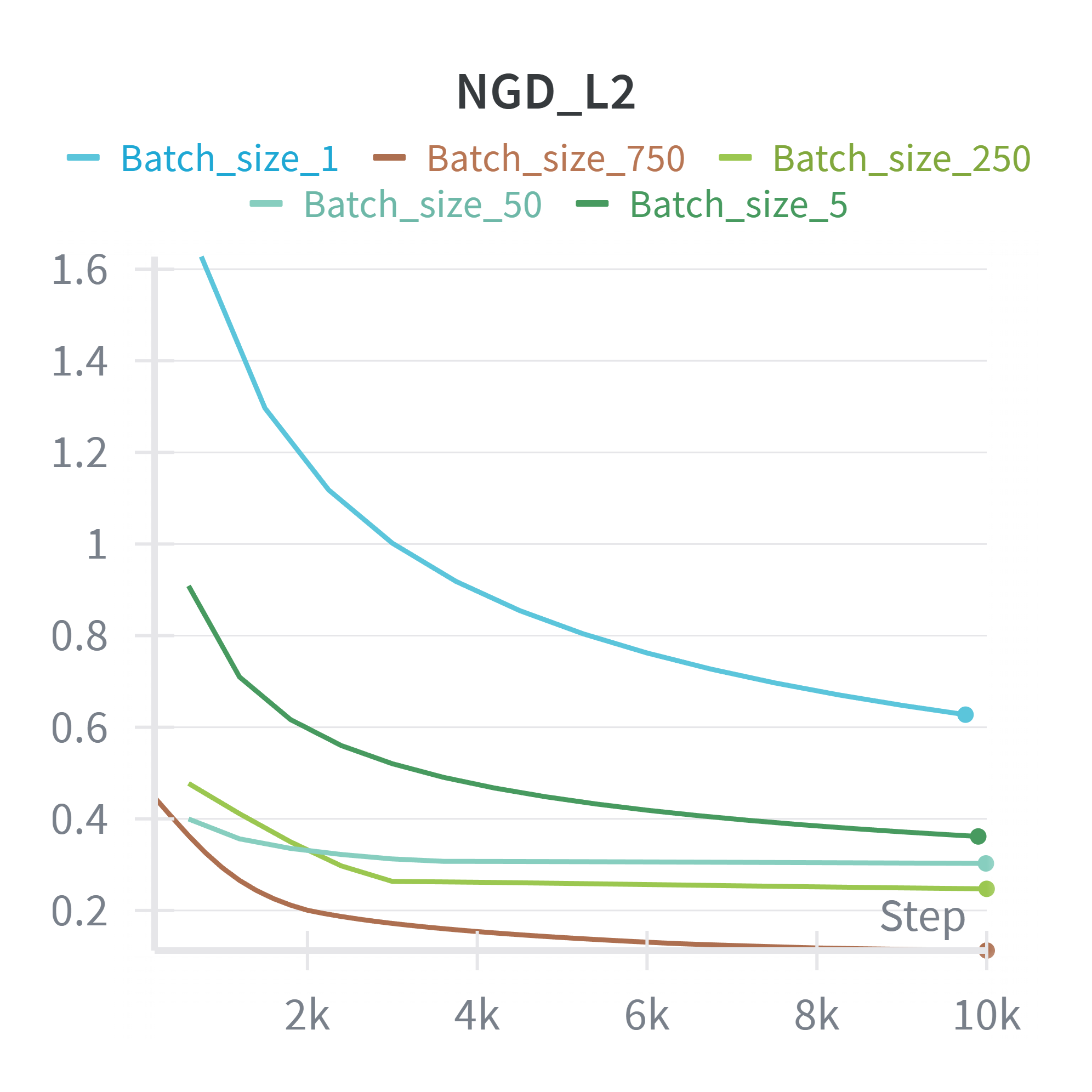}%
  }\hfill
  \subfigure[Margin error of NucGD]{%
    \includegraphics[width=0.23\textwidth]{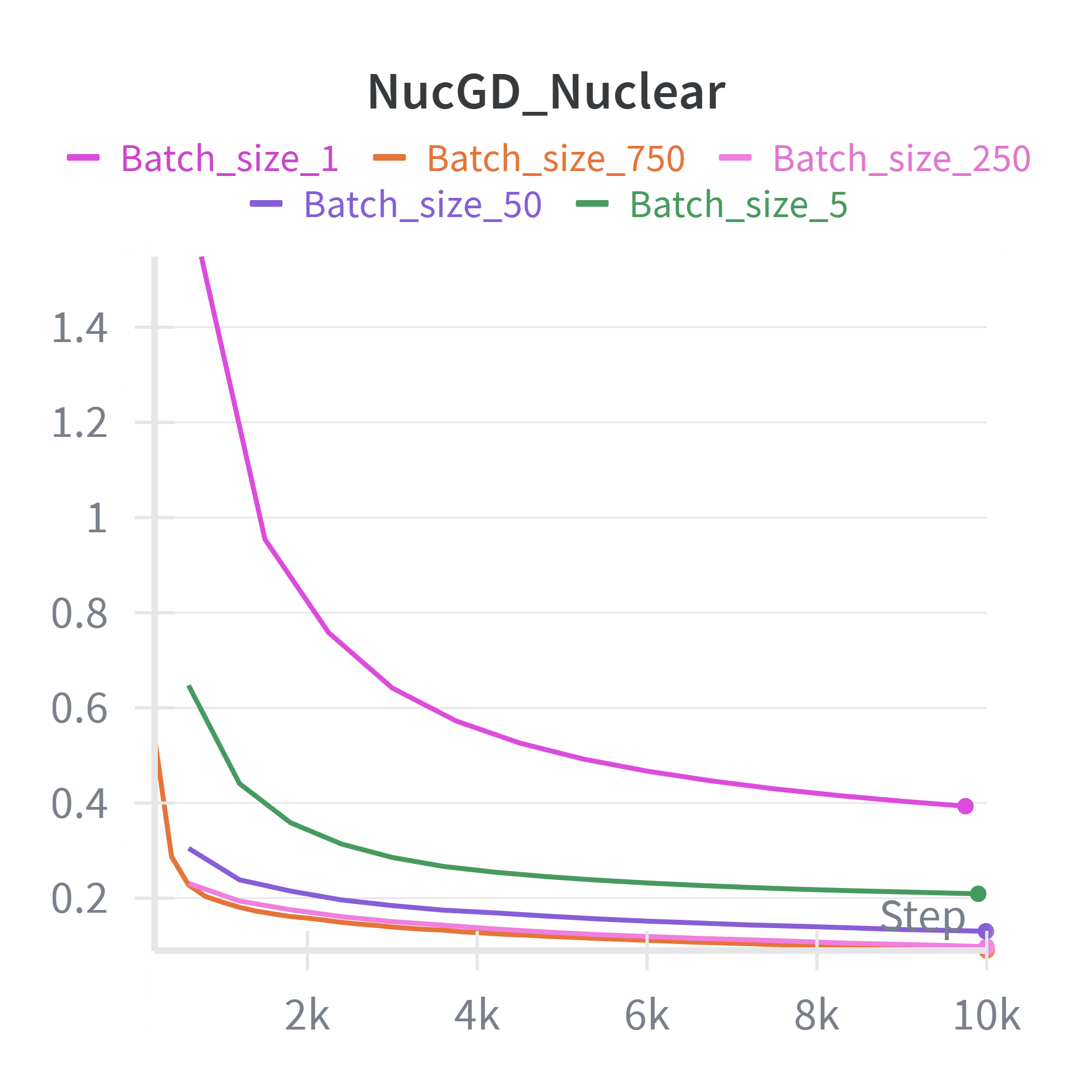}%
  }\hfill
  \subfigure[Margin error of SignGD]{%
    \includegraphics[width=0.23\textwidth]{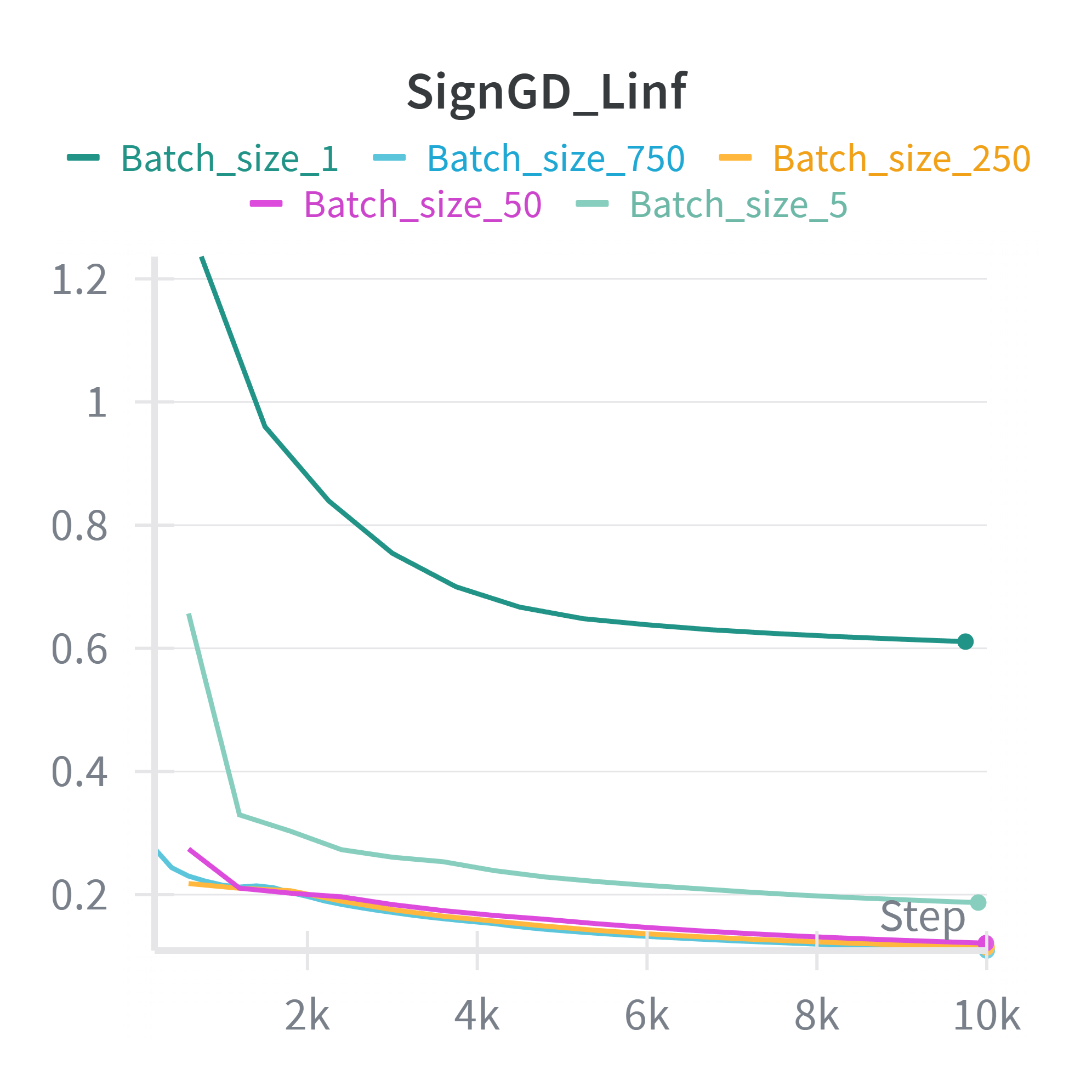}%
  }\\[2mm]
  \subfigure[Correlation of Muon]{%
    \includegraphics[width=0.23\textwidth]{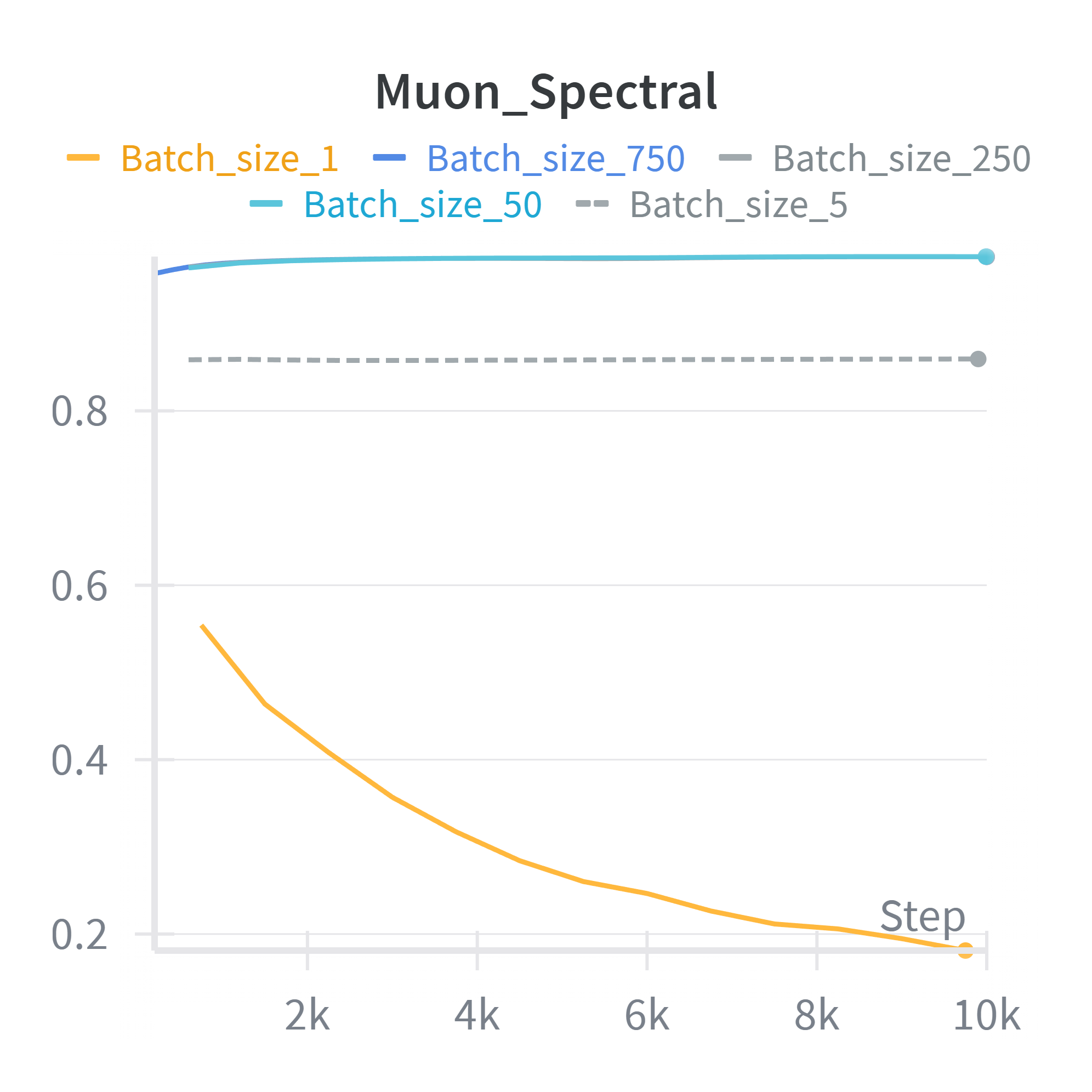}%
  }\hfill
  \subfigure[Correlation of NGD]{%
    \includegraphics[width=0.23\textwidth]{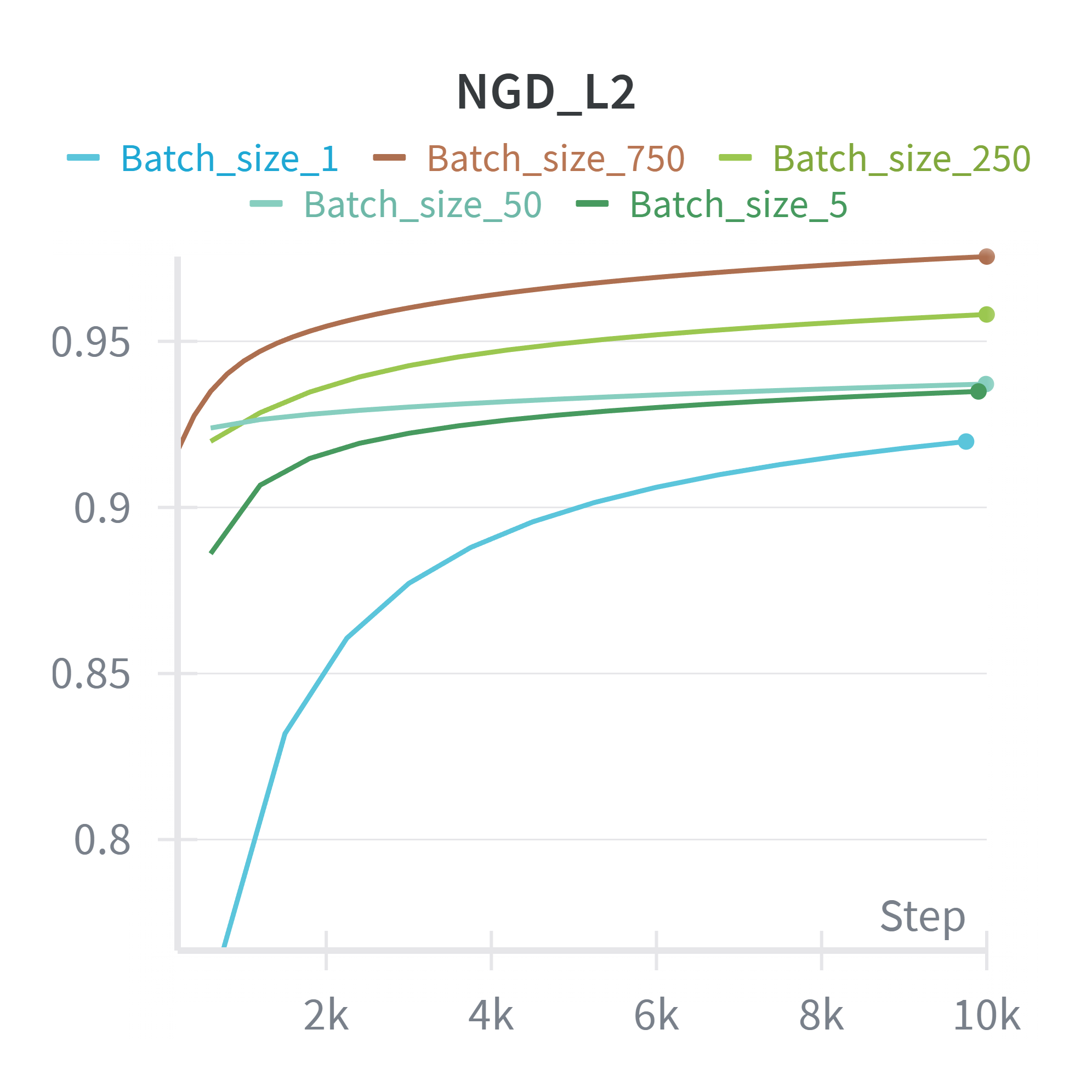}%
  }\hfill
  \subfigure[Correlation of NucGD]{%
    \includegraphics[width=0.23\textwidth]{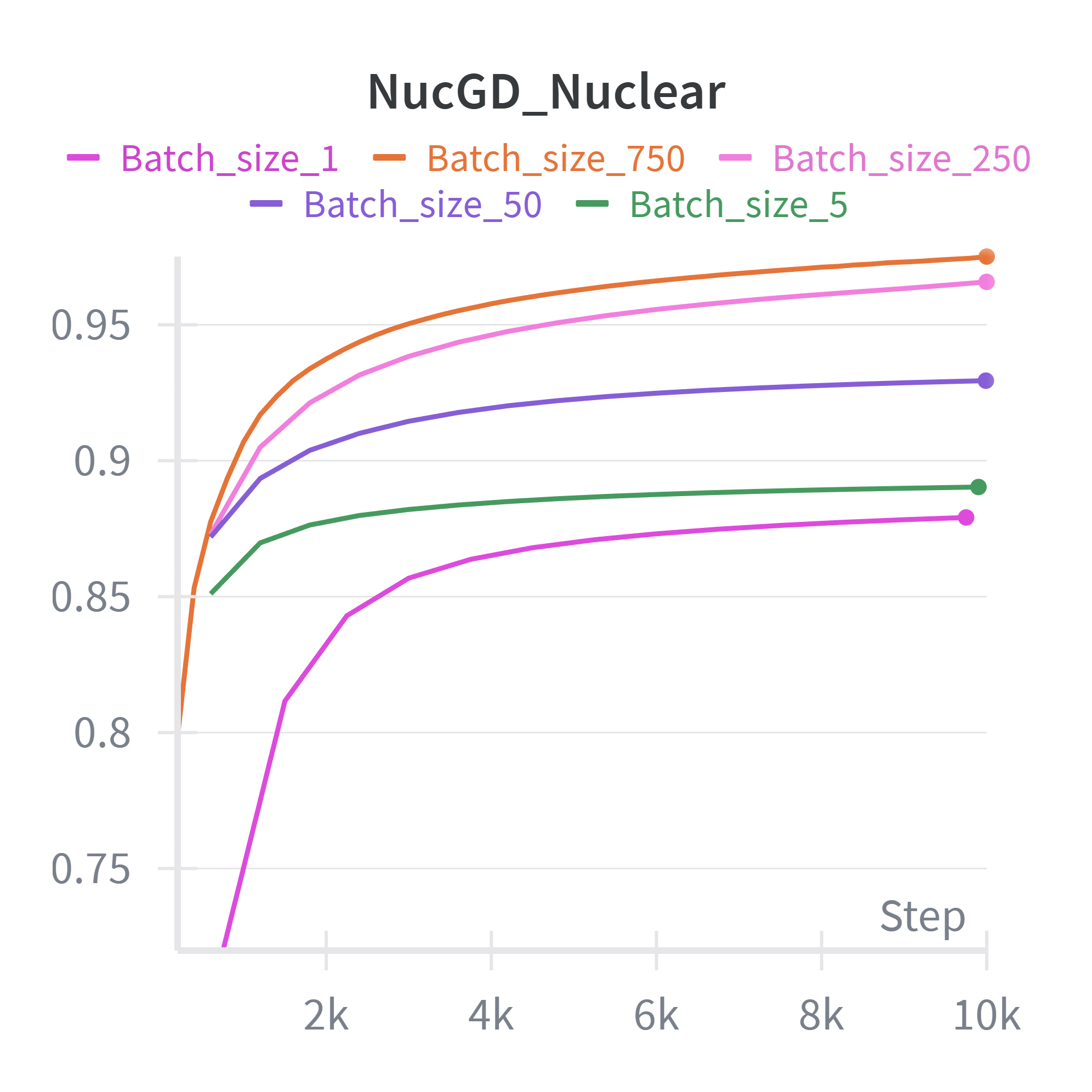}%
  }\hfill
  \subfigure[Correletion of SignGD]{%
    \includegraphics[width=0.23\textwidth]{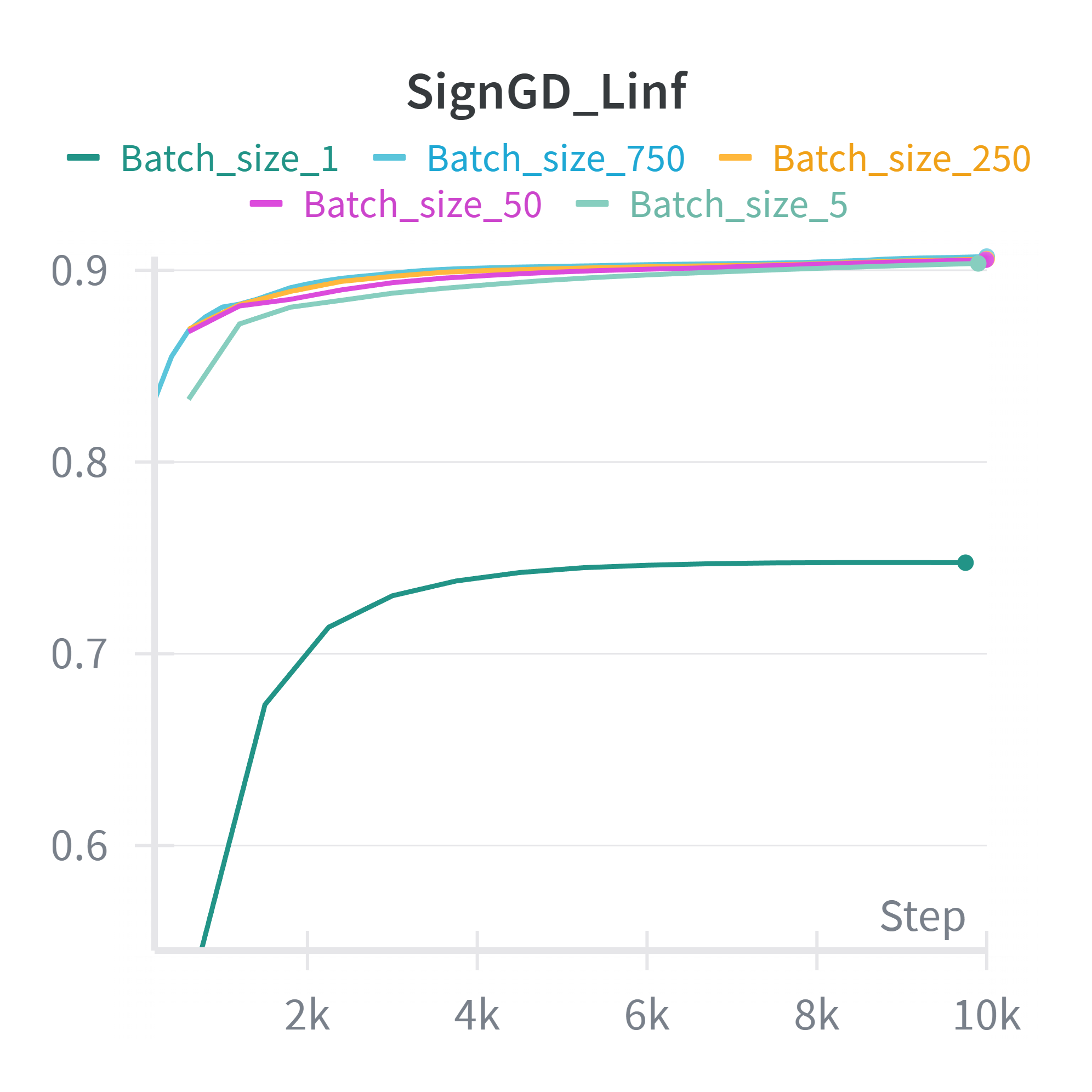}%
  }

  \caption{Effect of noise controlled by batch size $B$}
  \label{fig:noise}
\end{figure}

\paragraph{Results of (b).}
Figures~\ref{fig:momentum_minibatch} depict the variations in correlation and normalized margin with respect to the momentum weight $\mu$ with $B=50$. The results show that increasing $\mu$ significantly impacts the implicit bias: it leads to a lower normalized margin $\widehat{\gamma}$ while increasing the correlation.

\begin{figure}[H]
  \centering
  \subfigure[Margin errer of Muon]{%
    \includegraphics[width=0.23\textwidth]{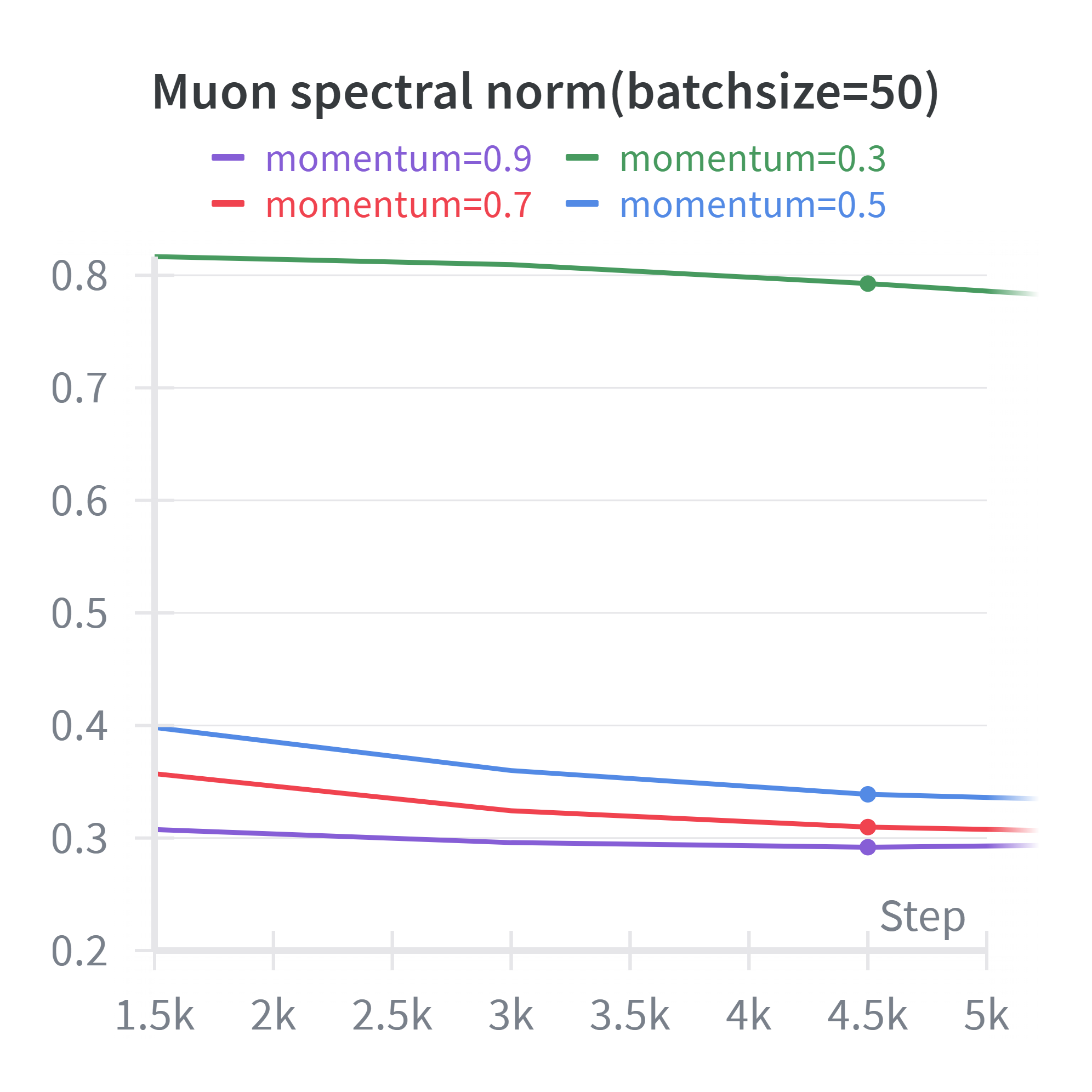}%
  }\hfill
  \subfigure[Margin errer of NGD]{%
    \includegraphics[width=0.23\textwidth]{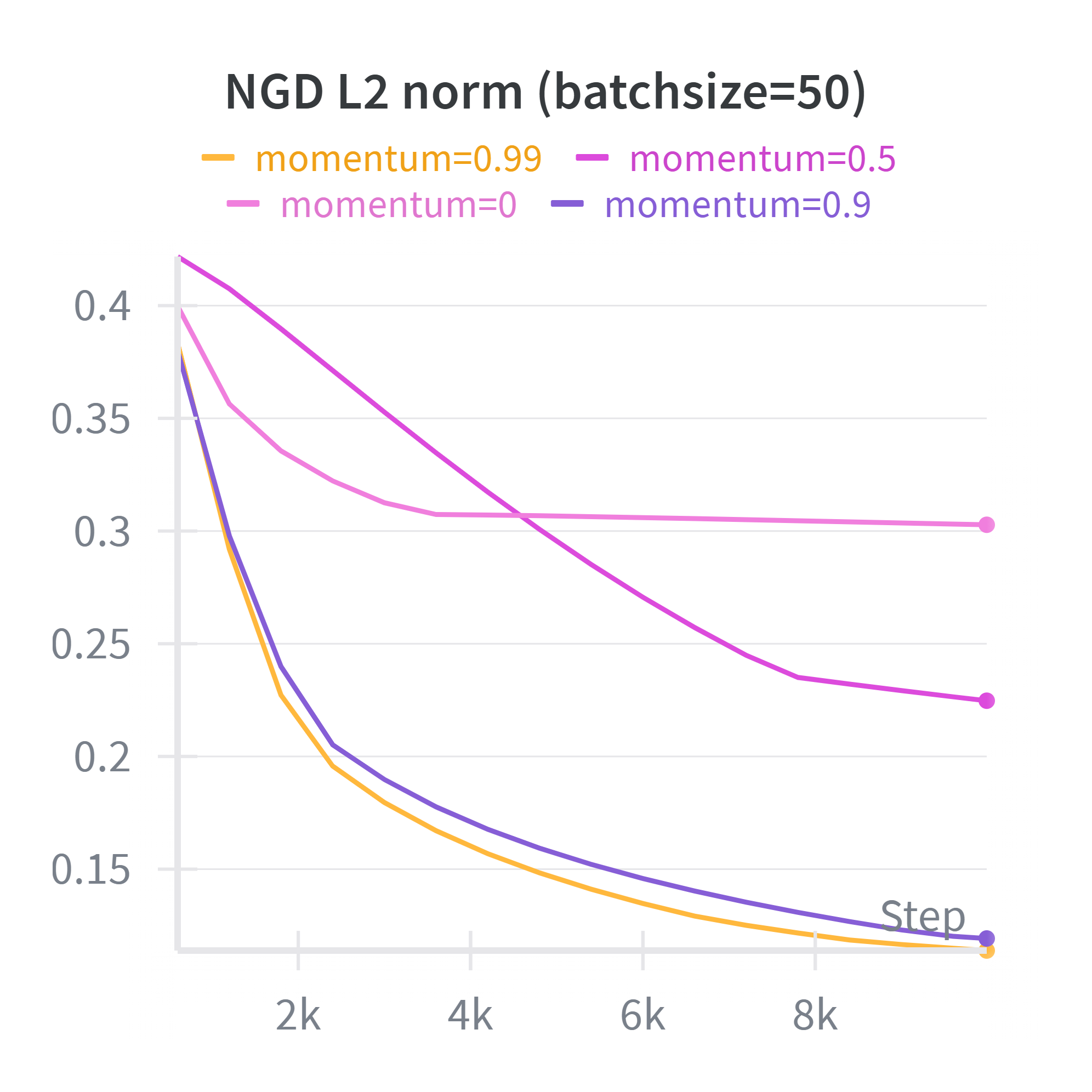}%
  }\hfill
  \subfigure[Margin errer of NucGD]{%
    \includegraphics[width=0.23\textwidth]{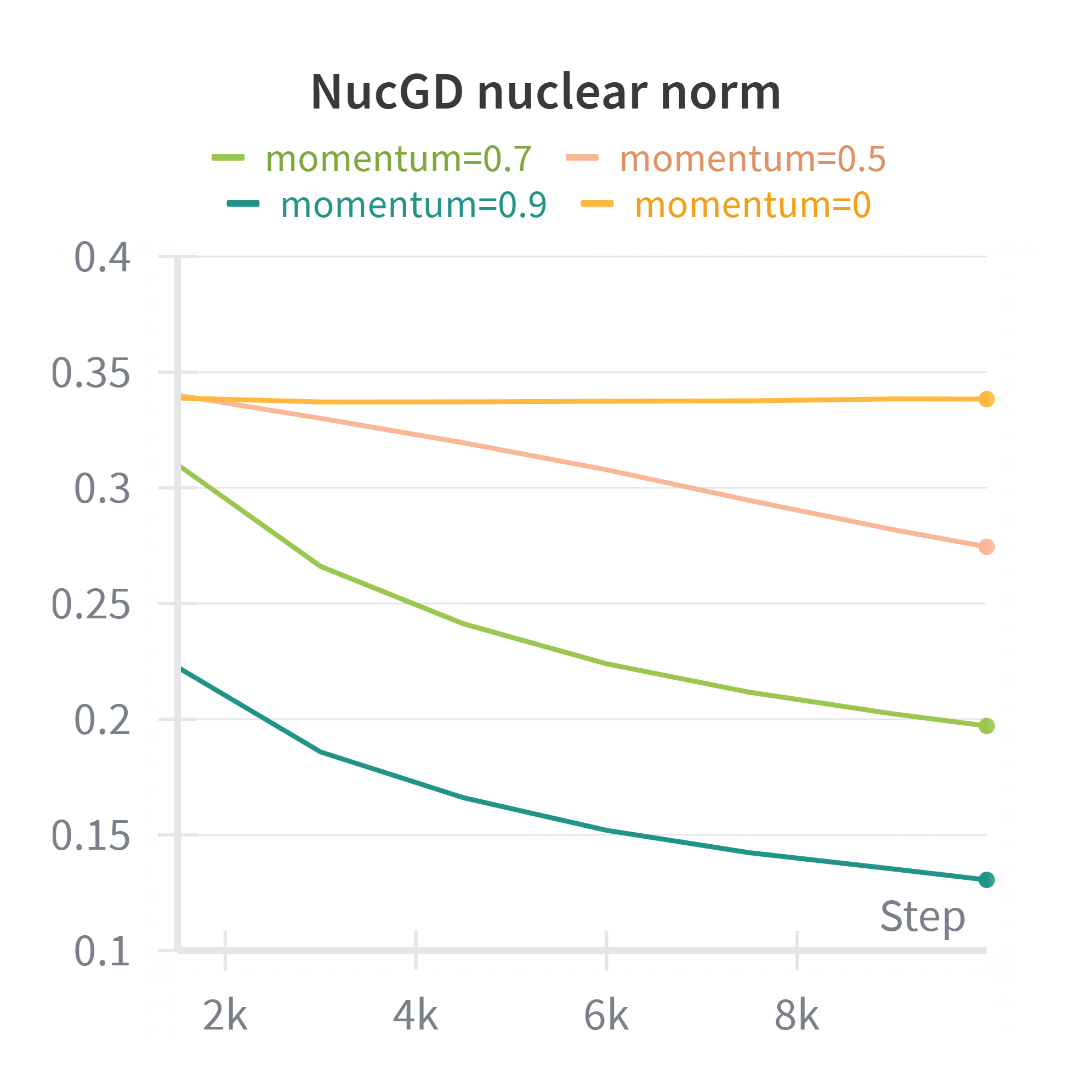}%
  }\hfill
  \subfigure[Margin errer of SignGD]{%
    \includegraphics[width=0.23\textwidth]{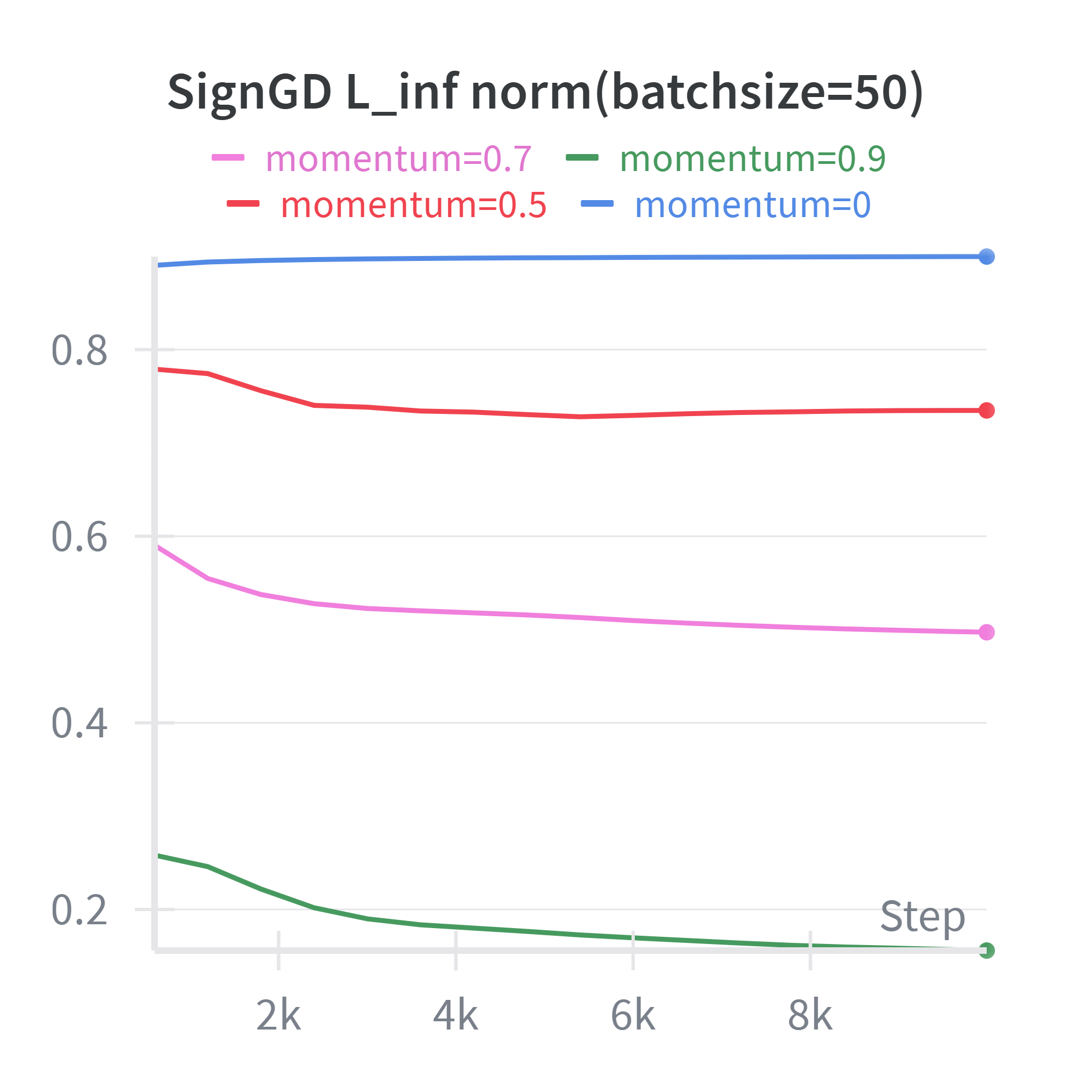}%
  }\hfill
  \\[2mm]
    \subfigure[Correlation of Muon]{%
    \includegraphics[width=0.23\textwidth]{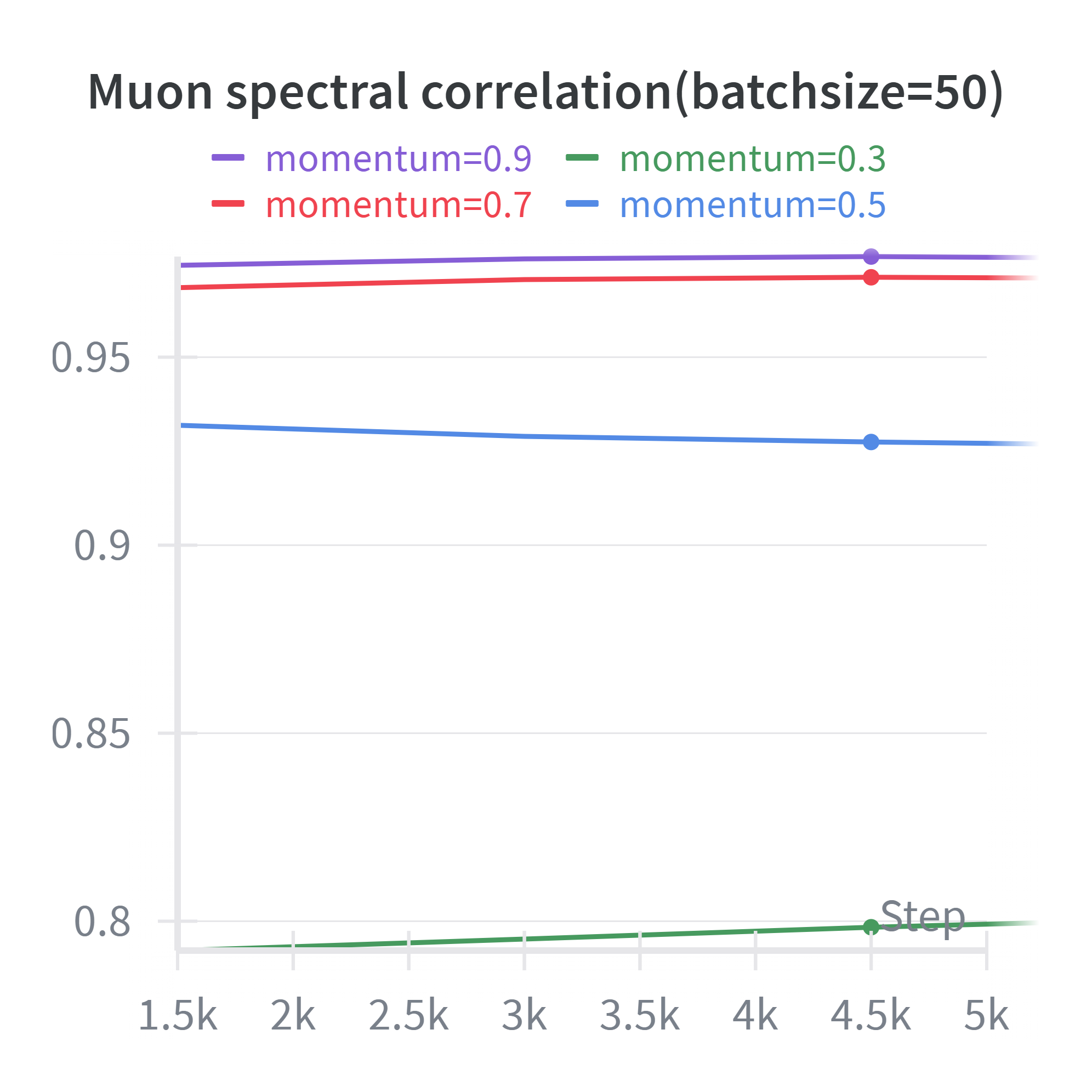}%
  }\hfill
  \subfigure[Correlation of NGD]{%
    \includegraphics[width=0.23\textwidth]{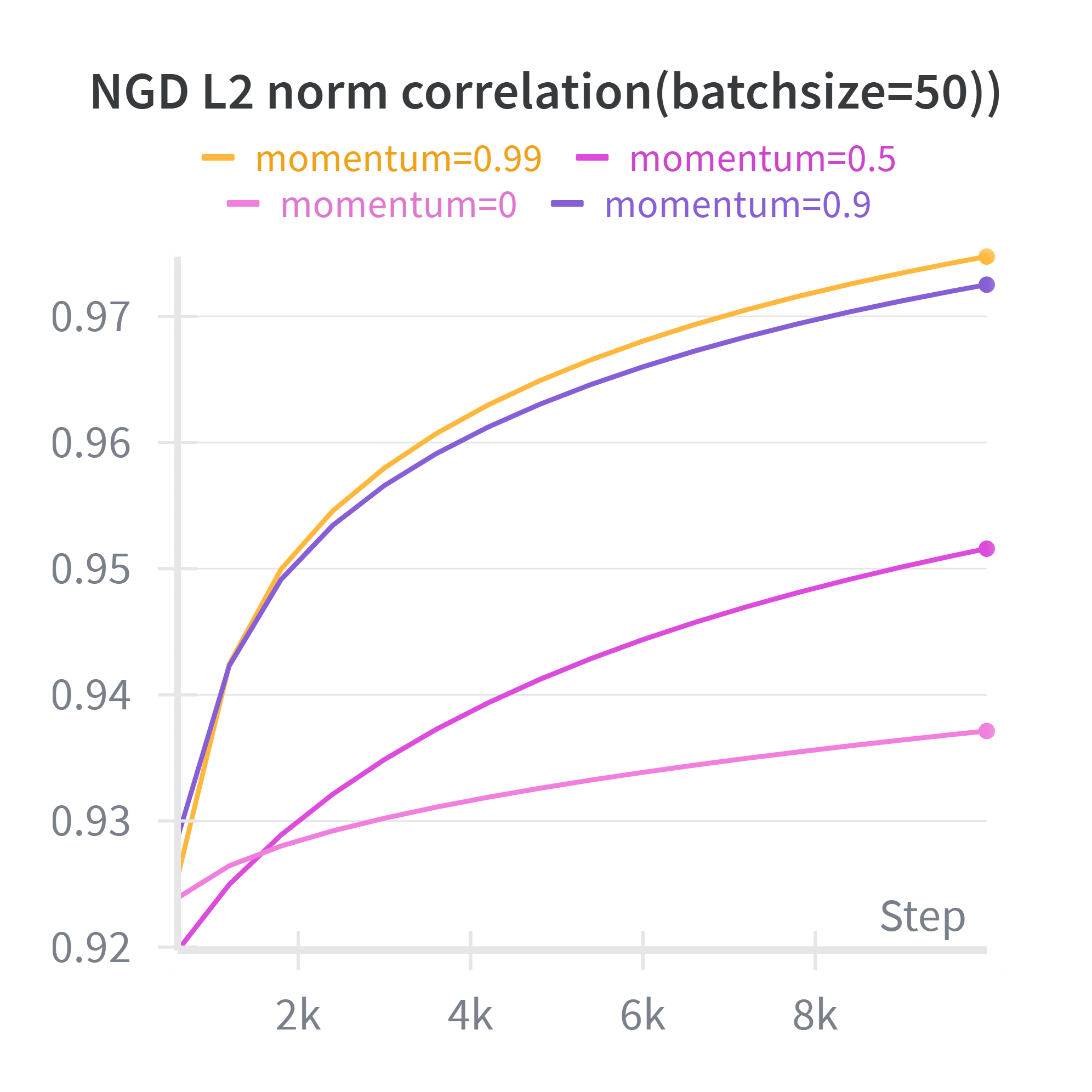}%
  }
  \subfigure[Correlation of NucGD]{%
    \includegraphics[width=0.23\textwidth]{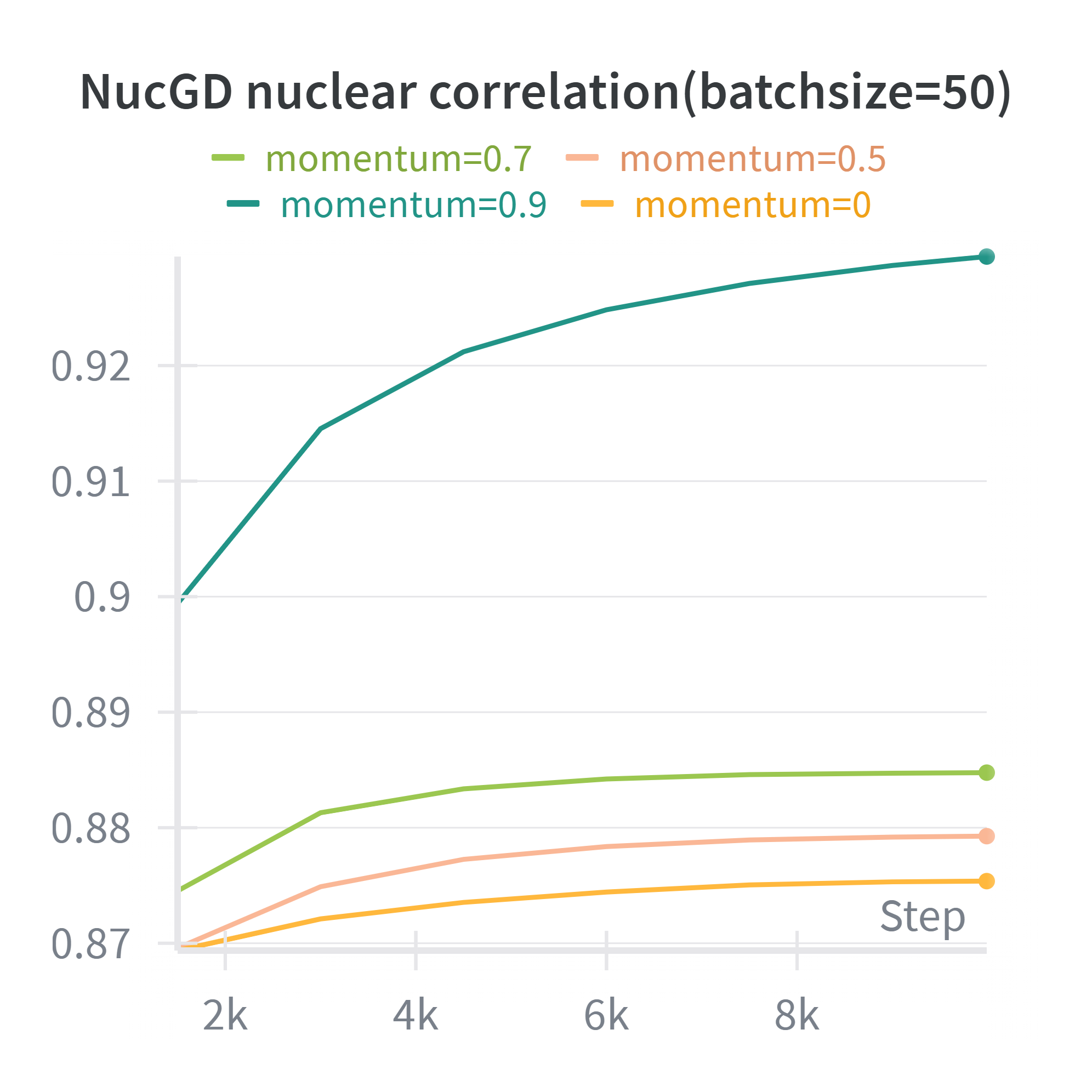}%
  }\hfill
  \subfigure[Coffelation of SignGD]{%
    \includegraphics[width=0.23\textwidth]{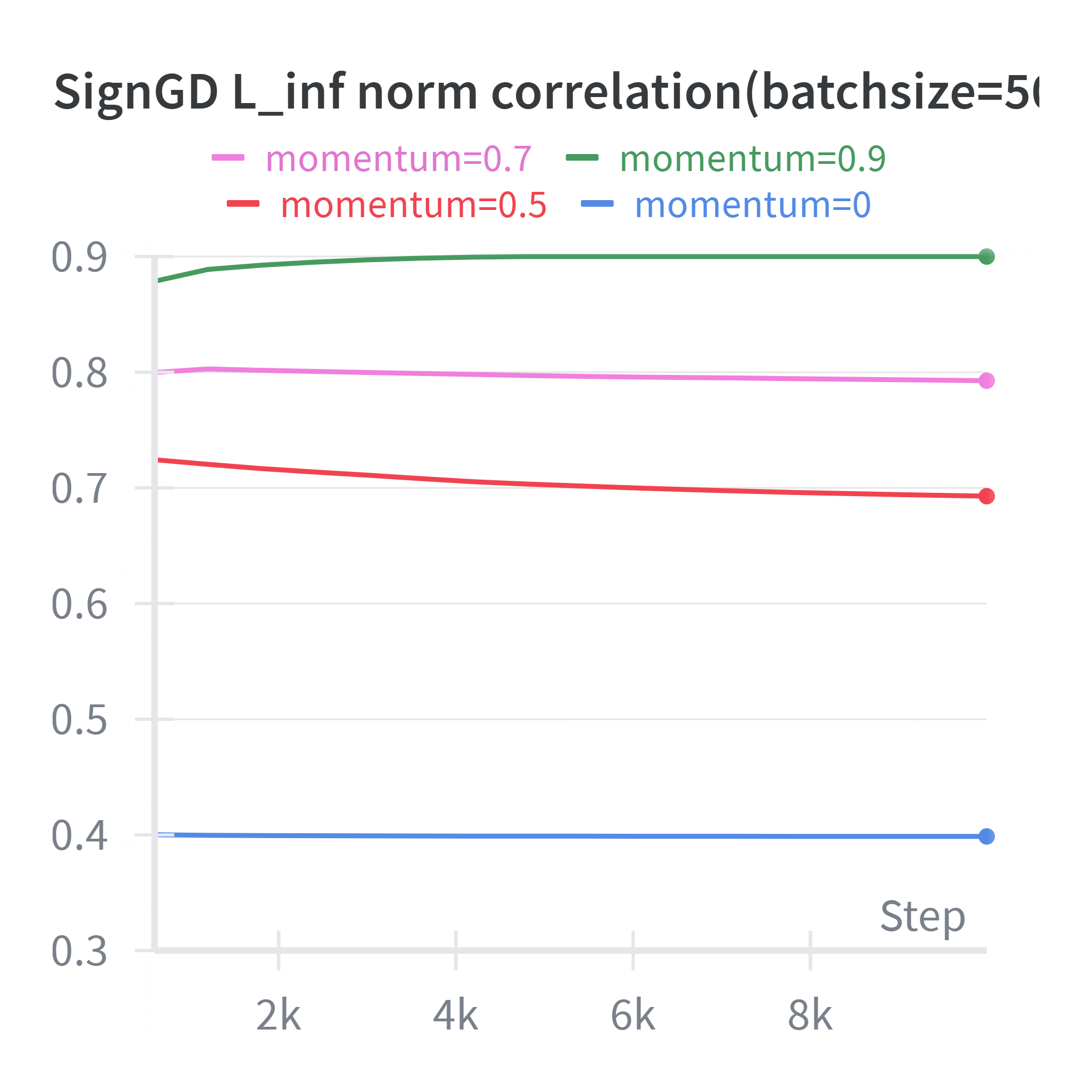}%
  }
  \caption{Effect of momentum weight $\mu$, \textbf{mini-batch} training}
  \label{fig:momentum_minibatch}
\end{figure}
This phenomenon is exclusive to the mini-batch setting. In Figure~\ref{fig:momentum_fullbatch}, we visualize the performance with different $\mu$ under full-batch gradient descent and observe that the impact is negligible: the implicit bias is significant under all levels of $\mu$. This suggests that $\mu$ influences implicit bias primarily by modulating gradient noise; consequently, in the absence of noise (full-batch), its effect diminishes.
\begin{figure}[H]
  \centering
  \begin{minipage}[b]{0.24\textwidth}
    \centering
    \includegraphics[width=\linewidth]{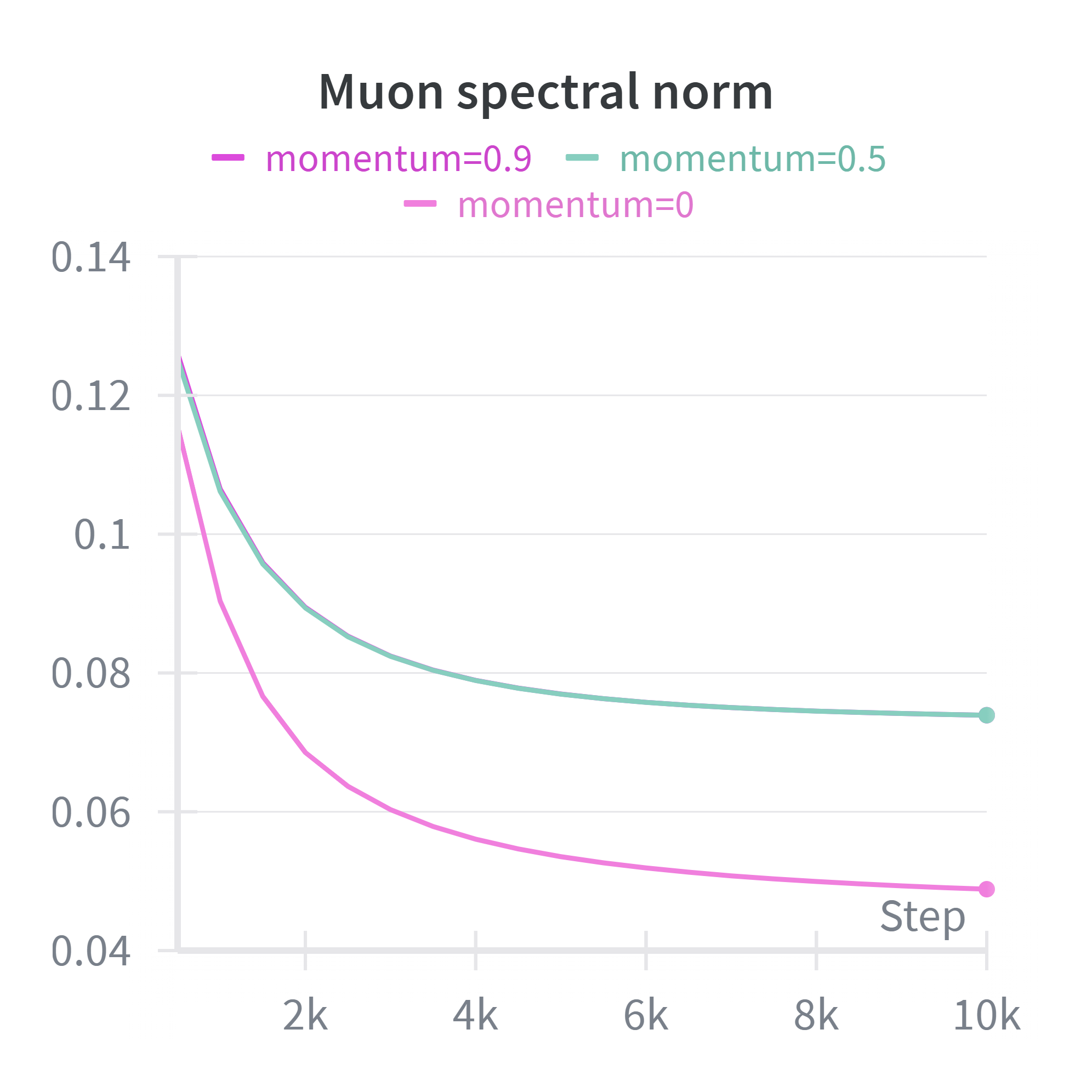}
    \centerline{\footnotesize (a) Muon: $\widehat{\gamma}$}
  \end{minipage}%
  \hfill
  \begin{minipage}[b]{0.24\textwidth}
    \centering
    \includegraphics[width=\linewidth]{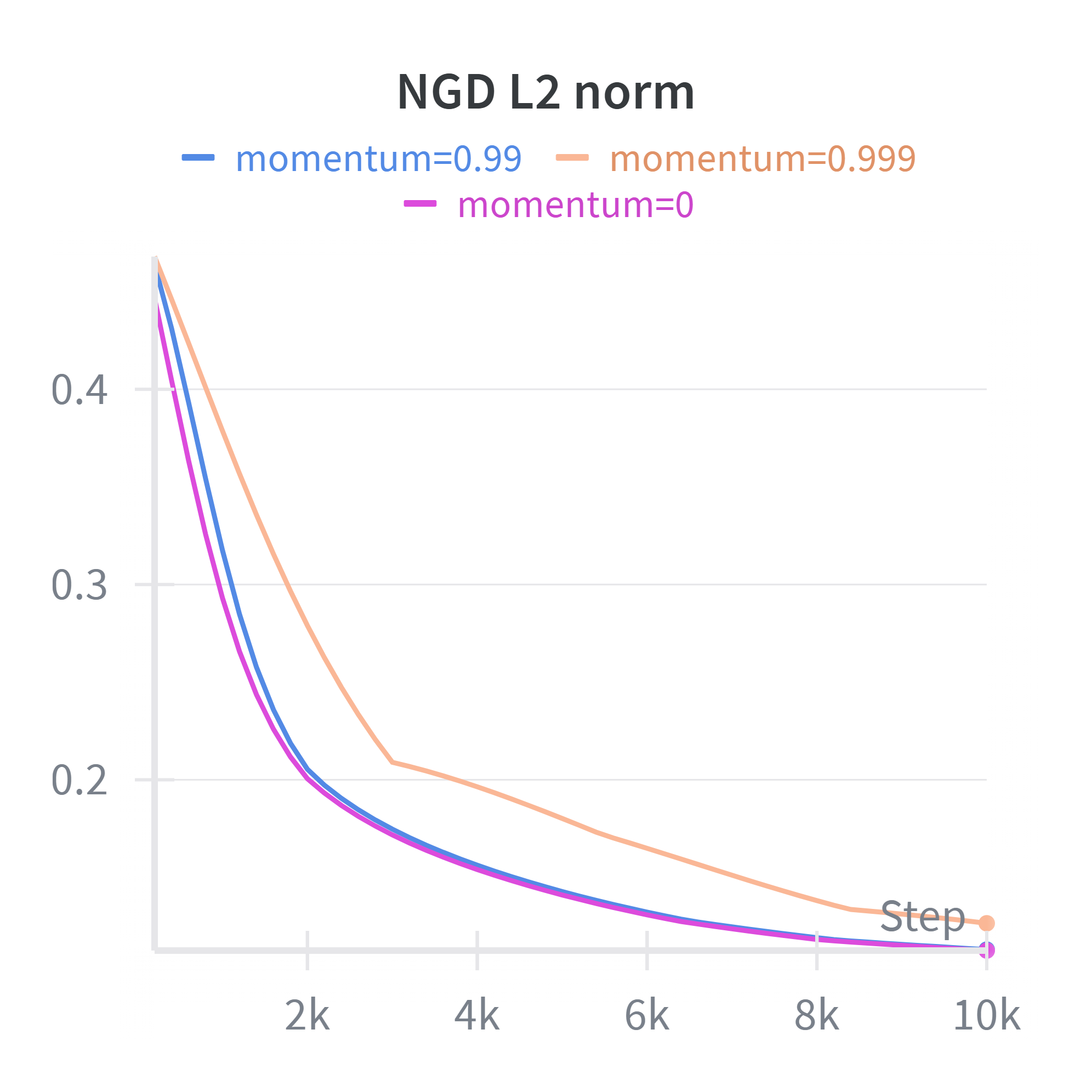}
    \centerline{\footnotesize (b) NGD: $\widehat{\gamma}$}
  \end{minipage}%
  \hfill
  \begin{minipage}[b]{0.24\textwidth}
    \centering
    \includegraphics[width=\linewidth]{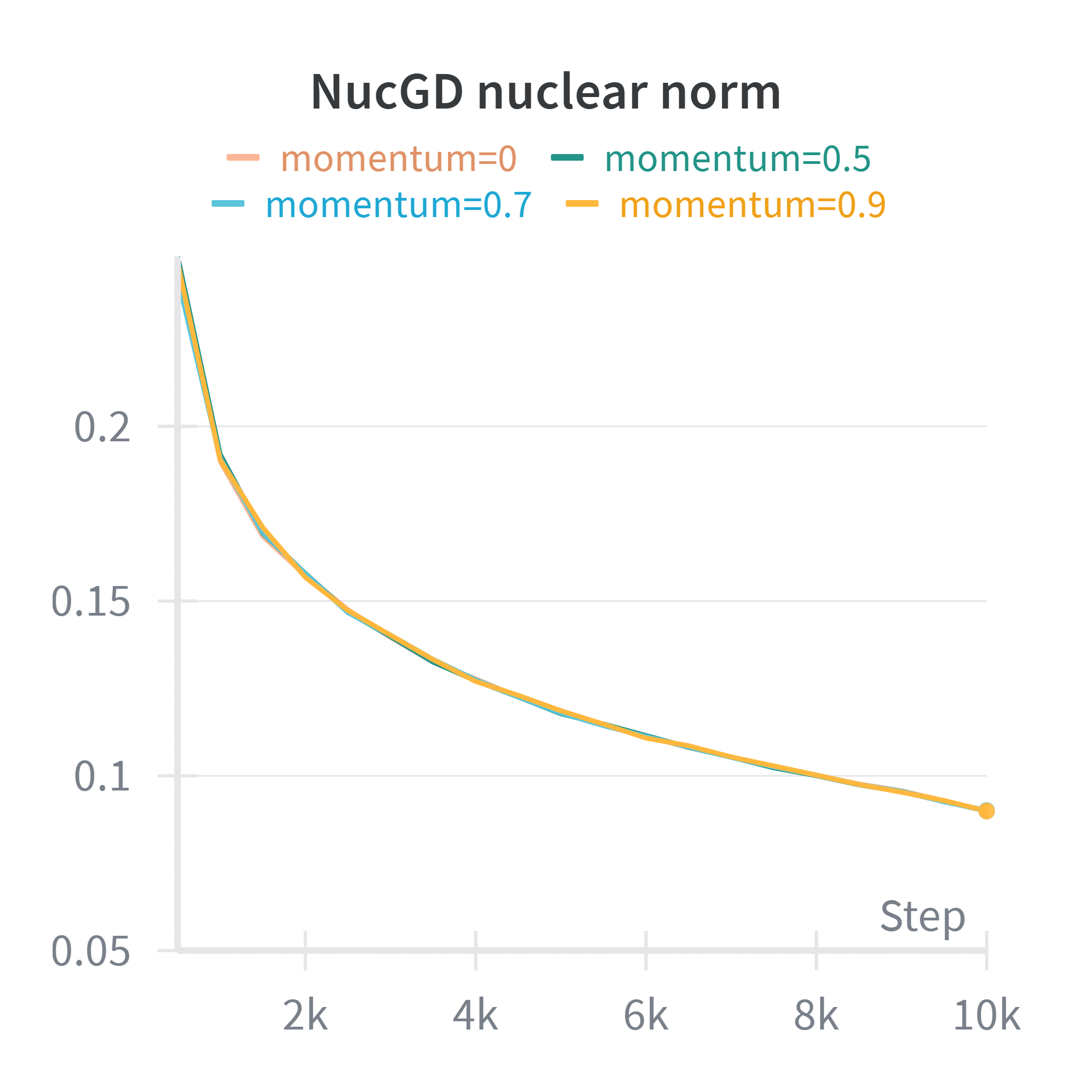}
    \centerline{\footnotesize (c) NucGD: $\widehat{\gamma}$}
  \end{minipage}%
  \hfill
  \begin{minipage}[b]{0.24\textwidth}
    \centering
    \includegraphics[width=\linewidth]{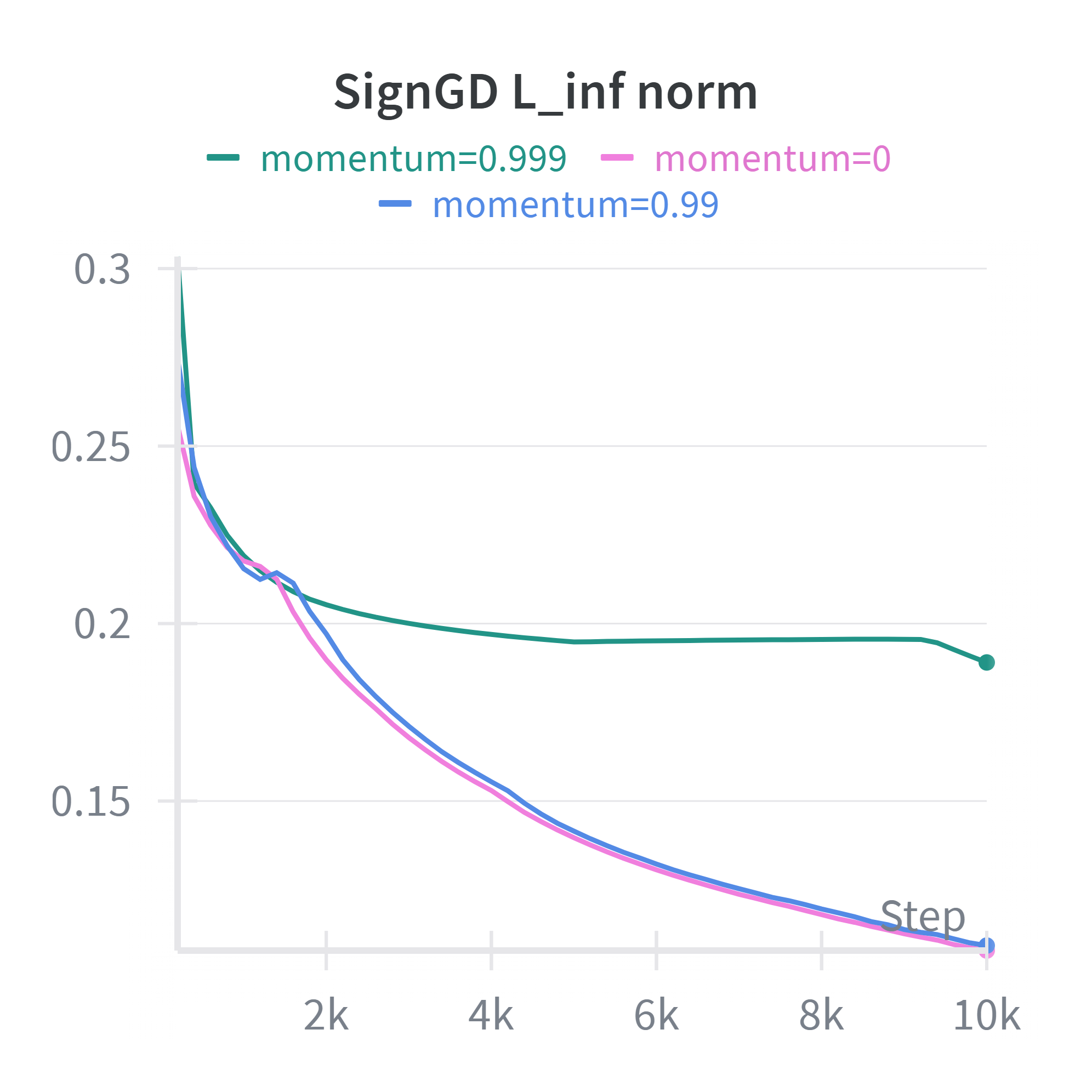}
    \centerline{\footnotesize (d) SignGD: $\widehat{\gamma}$}
  \end{minipage}

  \caption{Effect of momentum weight $\mu$, \textbf{full-batch} training}
  \label{fig:momentum_fullbatch}
\end{figure}
\vspace{-1mm}
\section{Conclusion and Future Work}

In this work, we studied the implicit bias of gradient-based optimization algorithms for multiclass linear classification on separable data through the lens of the Normalized Steepest Descent (NSD) framework, and further investigated the effect of gradient noise induced by mini-batch sampling and momentum.

Building on the theoretical analysis of Fan et al.~\cite{fan2025implicit}, we unified a broad class of optimization methods—including NGD, SignGD, and matrix-preconditioned methods such as Muon—by interpreting them as steepest descent under different norm constraints. Under this unified view, we empirically verified that NSD algorithms converge in direction to the corresponding norm-induced maximum-margin solutions.

Motivated by the benefits of low-rank structure, we extended the NSD framework to the nuclear norm setting and proposed a new algorithm, \NucGD. By exploiting the duality between the nuclear and spectral norms, we showed that \NucGD\ updates along the leading singular directions of the momentum matrix. To avoid the computational cost of full SVD, we further derived an efficient implementation based on asynchronous power iteration.

Extensive experiments demonstrate that \NucGD\ exhibits a clear implicit bias toward the maximum nuclear norm margin solution. Compared with NSD variants under $\ell_2$, $\ell_\infty$, and spectral norms, \NucGD\ achieves stronger alignment with the nuclear norm max-margin direction and produces solutions with significantly enhanced low-rank structure. These results indicate that \NucGD\ can effectively recover the benefits of nuclear norm regularization through implicit optimization dynamics.
Overall, our work extends implicit bias analysis to nuclear norm geometries and provides both theoretical and empirical support for low-rank–oriented normalized steepest descent.

As future work, it would be natural to extend the analysis of \NucGD\ beyond linear models to \emph{deep architectures} and get more substantial generalization results. Meanwhile, as \NucGD\ has similar update logic with normalized subspace optimizers such as\cite{zhao2024galore}, a theoretical convergence guarantee under basic stochastic non-convex setting is also expected, which may better characterize its convergence rate and rank dynamics in stochastic settings. Empirically, exploring tighter connections between \NucGD\ and practical low-rank training techniques, such as parameter-efficient fine-tuning and communication-efficient algorithms in distributed learning also presents a promising direction.
\vspace{-5mm}
\newpage
\textbf{References}
\bibliographystyle{unsrt}
\bibliography{references}

\begin{thebibliography}{10}

\bibitem{soudry2018implicit}
Daniel Soudry, Elad Hoffer, Mor~Shpigel Nacson, Suriya Gunasekar, and Nathan Srebro.
\newblock The implicit bias of gradient descent on separable data.
\newblock {\em Journal of Machine Learning Research}, 19(70):1--57, 2018.

\bibitem{ravi2024implicit}
Hrithik Ravi, Clay Scott, Daniel Soudry, and Yutong Wang.
\newblock The implicit bias of gradient descent on separable multiclass data.
\newblock {\em Advances in Neural Information Processing Systems}, 37:81324--81359, 2024.

\bibitem{zhang2024implicit}
Chenyang Zhang, Difan Zou, and Yuan Cao.
\newblock The implicit bias of adam on separable data.
\newblock {\em Advances in Neural Information Processing Systems}, 37:23988--24021, 2024.

\bibitem{xie2024implicit}
Shuo Xie and Zhiyuan Li.
\newblock Implicit bias of adamw.
\newblock {\em arXiv preprint arXiv:2404.04454}, 2024.

\bibitem{liu2025muon}
Jingyuan Liu, Jianlin Su, Xingcheng Yao, Zhejun Jiang, Guokun Lai, Yulun Du, Yidao Qin, Weixin Xu, Enzhe Lu, Junjie Yan, et~al.
\newblock Muon is scalable for llm training.
\newblock {\em arXiv preprint arXiv:2502.16982}, 2025.

\bibitem{gupta2018shampoo}
Vineet Gupta, Tomer Koren, and Yoram Singer.
\newblock Shampoo: Preconditioned stochastic tensor optimization.
\newblock In {\em International Conference on Machine Learning}, pages 1842--1850. PMLR, 2018.

\bibitem{riabinin2025gluon}
Artem Riabinin, Egor Shulgin, Kaja Gruntkowska, and Peter Richt{\'a}rik.
\newblock Gluon: Making muon \& scion great again!(bridging theory and practice of lmo-based optimizers for llms).
\newblock {\em arXiv preprint arXiv:2505.13416}, 2025.

\bibitem{lau2025polargrad}
Tim Tsz-Kit Lau, Qi~Long, and Weijie Su.
\newblock Polargrad: A class of matrix-gradient optimizers from a unifying preconditioning perspective.
\newblock {\em arXiv preprint arXiv:2505.21799}, 2025.

\bibitem{fan2025implicit}
Chen Fan, Mark Schmidt, and Christos Thrampoulidis.
\newblock Implicit bias of spectral descent and muon on multiclass separable data.
\newblock {\em arXiv preprint arXiv:2502.04664}, 2025.

\bibitem{zhao2024galore}
Jiawei Zhao, Zhenyu Zhang, Beidi Chen, Zhangyang Wang, Anima Anandkumar, and Yuandong Tian.
\newblock Galore: Memory-efficient llm training by gradient low-rank projection.
\newblock {\em arXiv preprint arXiv:2403.03507}, 2024.

\bibitem{hazan2015beyond}
Elad Hazan, Kfir Levy, and Shai Shalev-Shwartz.
\newblock Beyond convexity: Stochastic quasi-convex optimization.
\newblock {\em Advances in neural information processing systems}, 28, 2015.

\bibitem{cutkosky2020momentum}
Ashok Cutkosky and Harsh Mehta.
\newblock Momentum improves normalized sgd.
\newblock In {\em International conference on machine learning}, pages 2260--2268. PMLR, 2020.

\bibitem{bernstein2018signsgd}
Jeremy Bernstein, Yu-Xiang Wang, Kamyar Azizzadenesheli, and Animashree Anandkumar.
\newblock signsgd: Compressed optimisation for non-convex problems.
\newblock In {\em International conference on machine learning}, pages 560--569. PMLR, 2018.

\bibitem{bernstein2024old}
Jeremy Bernstein and Laker Newhouse.
\newblock Old optimizer, new norm: An anthology.
\newblock {\em arXiv preprint arXiv:2409.20325}, 2024.

\bibitem{scarvelis2024nuclear}
Christopher Scarvelis and Justin~M Solomon.
\newblock Nuclear norm regularization for deep learning.
\newblock {\em Advances in Neural Information Processing Systems}, 37:116223--116253, 2024.

\bibitem{vogels2019powersgd}
Thijs Vogels, Sai~Praneeth Karimireddy, and Martin Jaggi.
\newblock Powersgd: Practical low-rank gradient compression for distributed optimization.
\newblock {\em Advances in Neural Information Processing Systems}, 32, 2019.

\end{thebibliography}
\newpage
\appendix
\section*{Appendix}
\section{Reproduction}\label{reproduction}
We reproduce the implicit bias reported by Fan et al.~\cite{fan2025implicit}(\NucGD\ added). The concise results are as follows, corresponding well with the normalized steepest descent framework.

\begin{figure}[H]
  \centering
  \begin{minipage}[b]{0.23\textwidth}
    \centering
    \includegraphics[width=\linewidth]{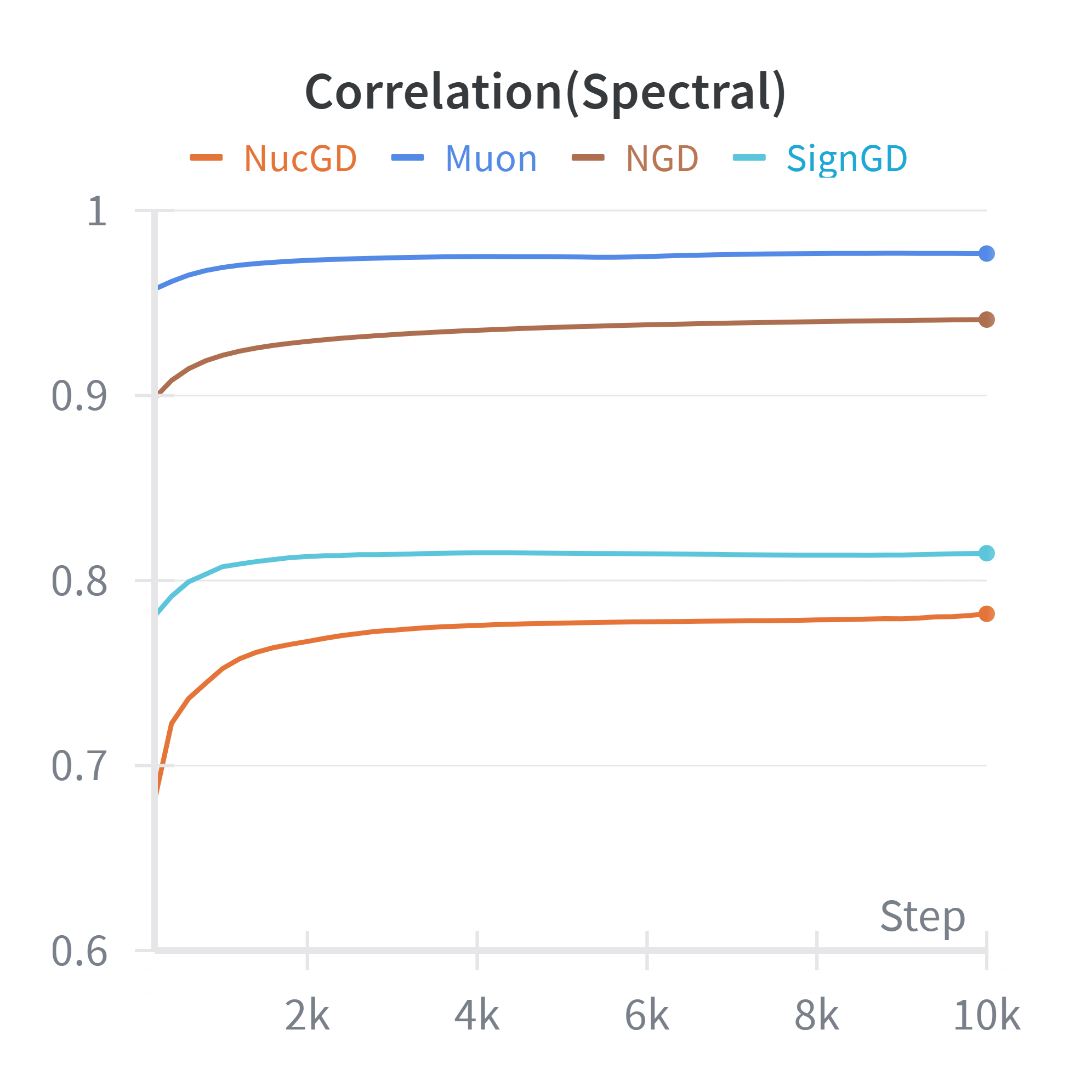}
    \centerline{}
  \end{minipage}%
  \hfill
  \begin{minipage}[b]{0.23\textwidth}
    \centering
    \includegraphics[width=\linewidth]{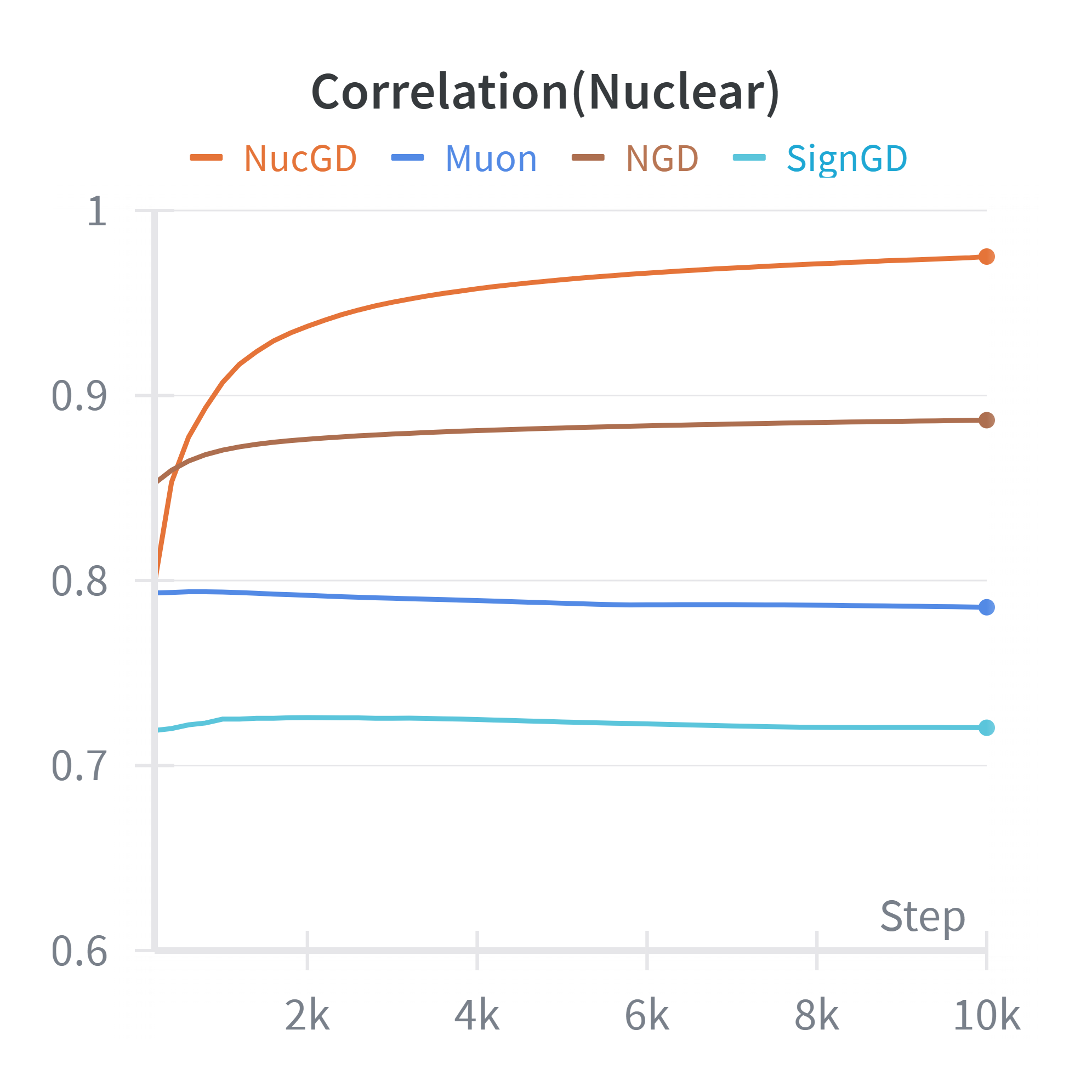}
    \centerline{}
  \end{minipage}%
  \hfill
  \begin{minipage}[b]{0.23\textwidth}
    \centering
    \includegraphics[width=\linewidth]{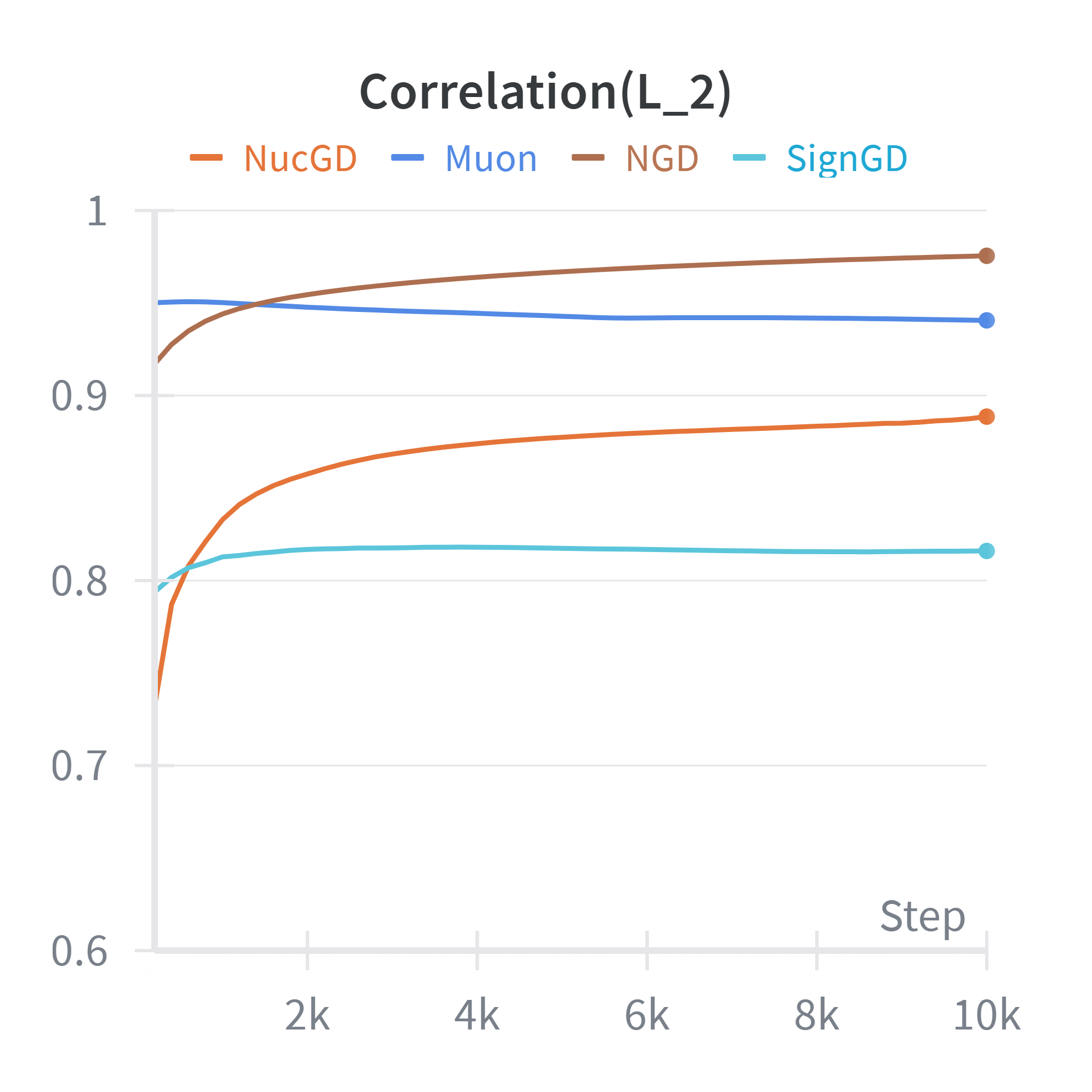}
    \centerline{}
  \end{minipage}%
  \hfill
  \begin{minipage}[b]{0.23\textwidth}
    \centering
    \includegraphics[width=\linewidth]{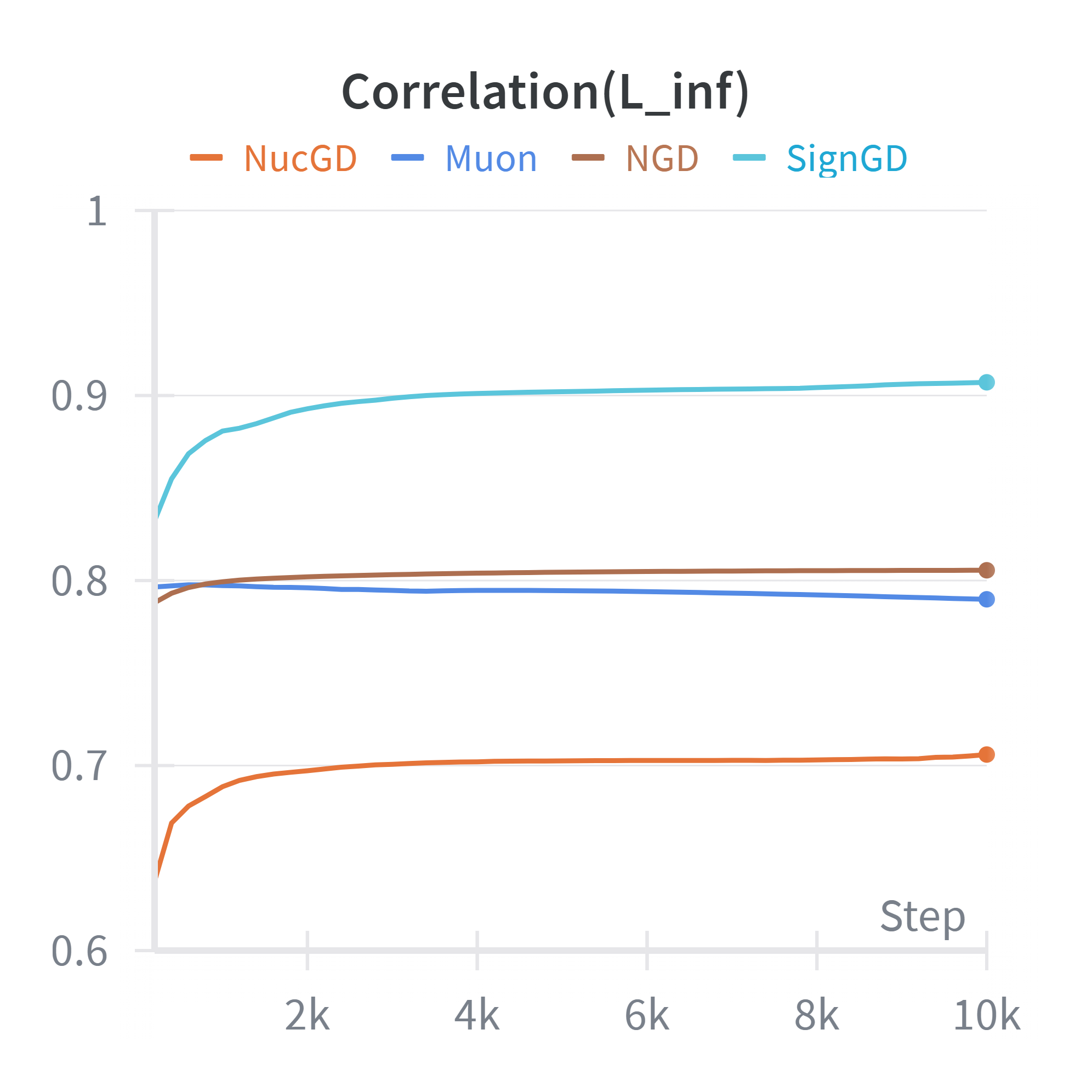}
    \centerline{}
  \end{minipage}

  \caption{Correlation between GD based solution and optimal solution corresponding to specific norm }
  \label{fig:corr}
\end{figure}

\begin{figure}[H]
  \centering
  \begin{minipage}[b]{0.23\textwidth}
    \centering
    \includegraphics[width=\linewidth]{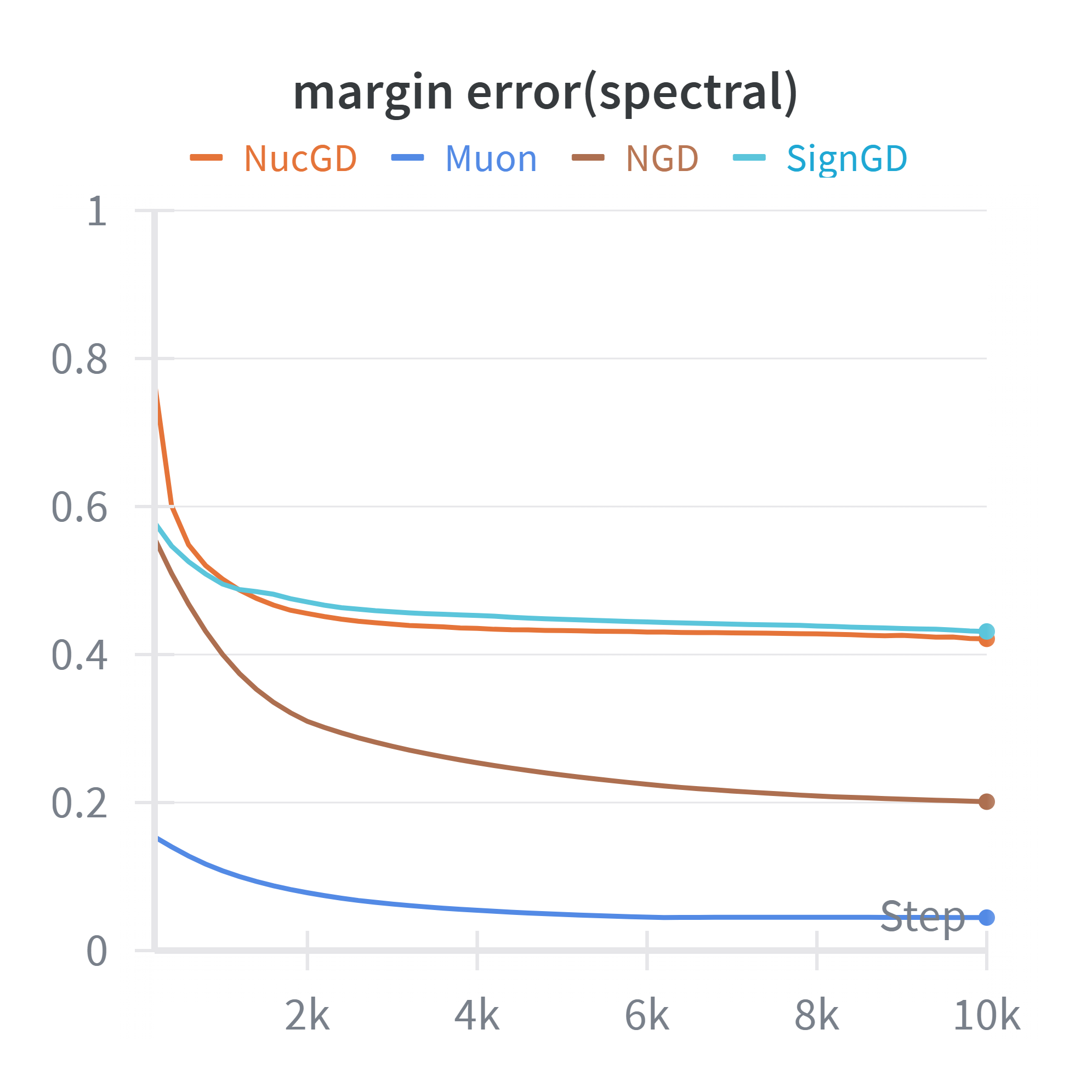}
    \centerline{}
  \end{minipage}%
  \hfill
  \begin{minipage}[b]{0.23\textwidth}
    \centering
    \includegraphics[width=\linewidth]{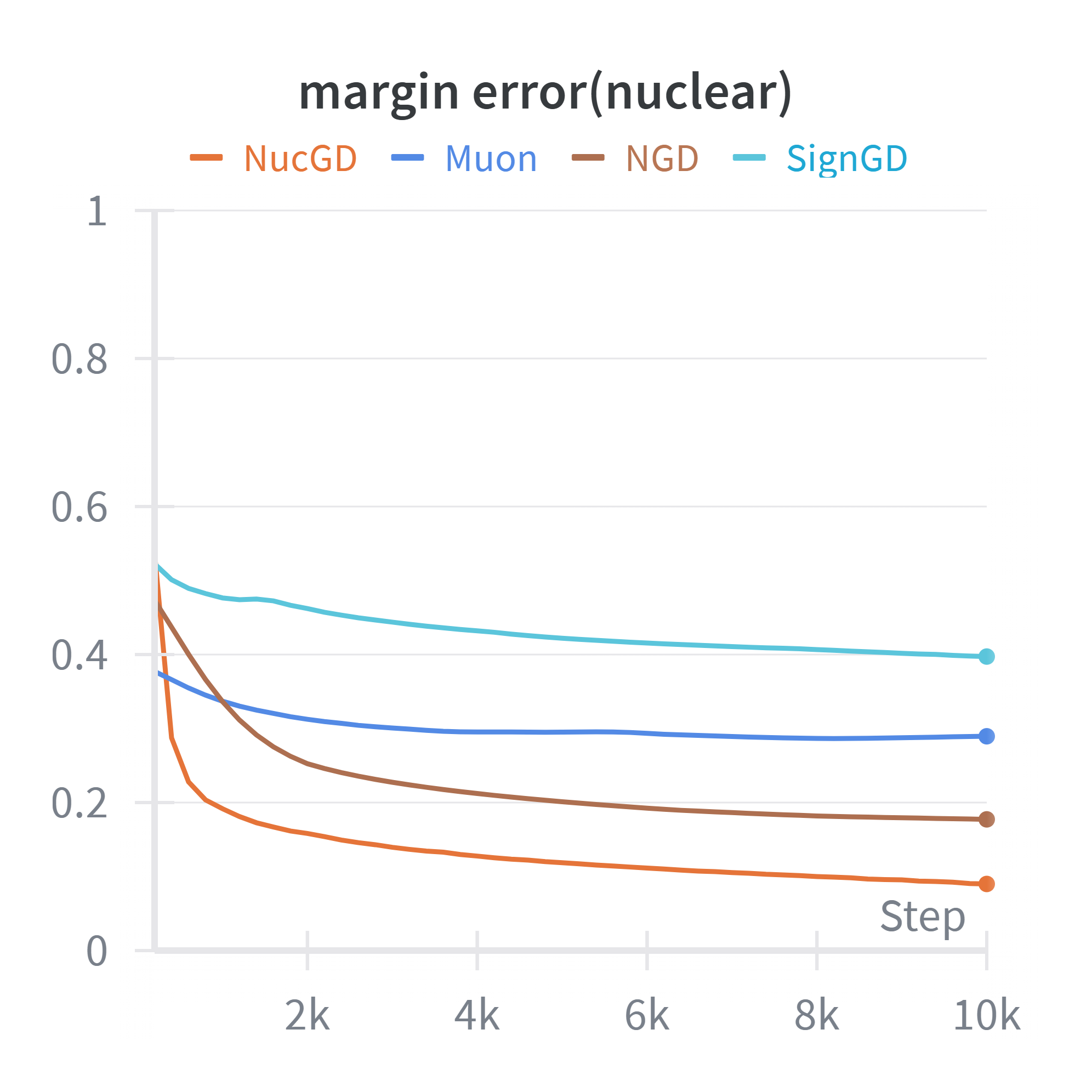}
    \centerline{}
  \end{minipage}%
  \hfill
  \begin{minipage}[b]{0.23\textwidth}
    \centering
    \includegraphics[width=\linewidth]{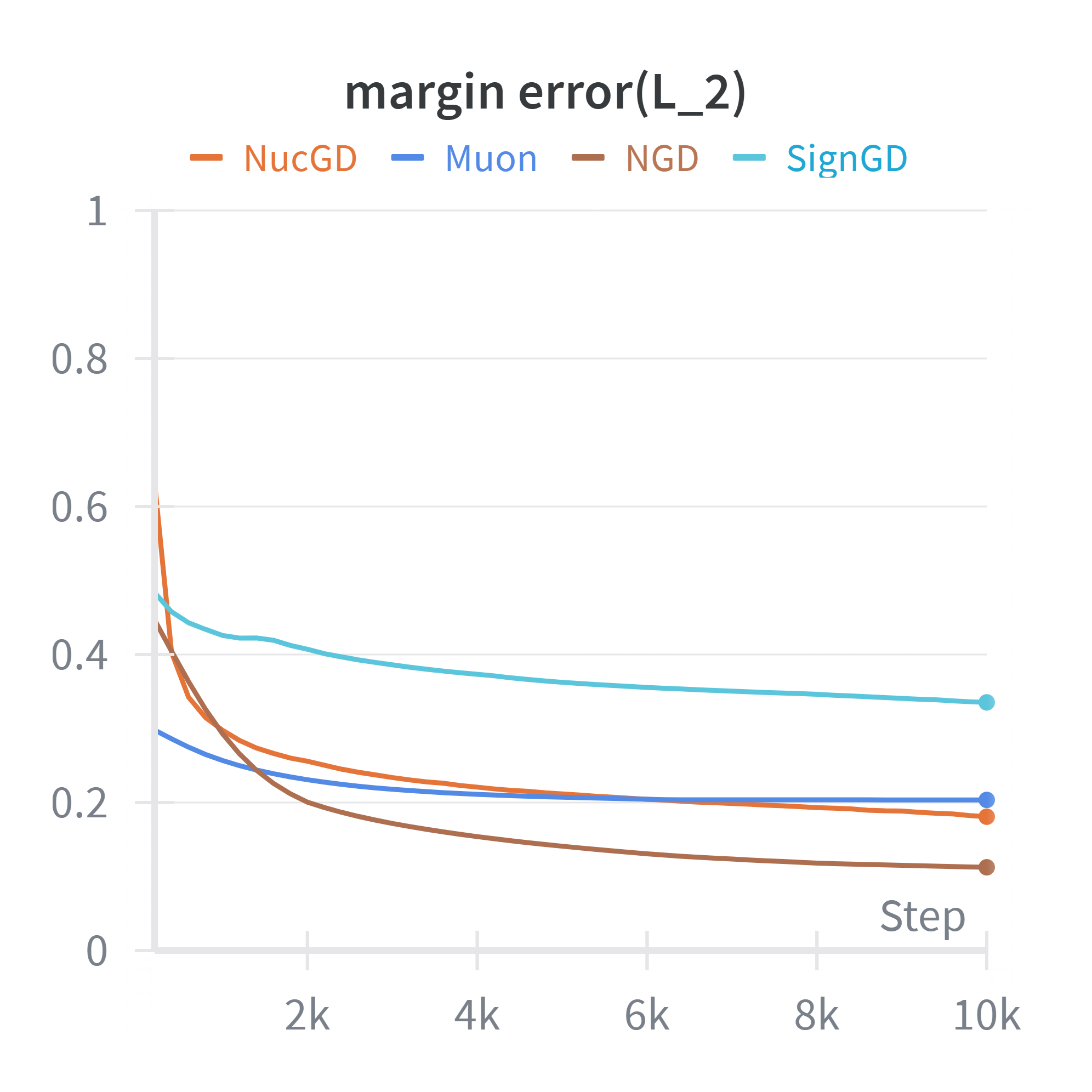}
    \centerline{}
  \end{minipage}%
  \hfill
  \begin{minipage}[b]{0.23\textwidth}
    \centering
    \includegraphics[width=\linewidth]{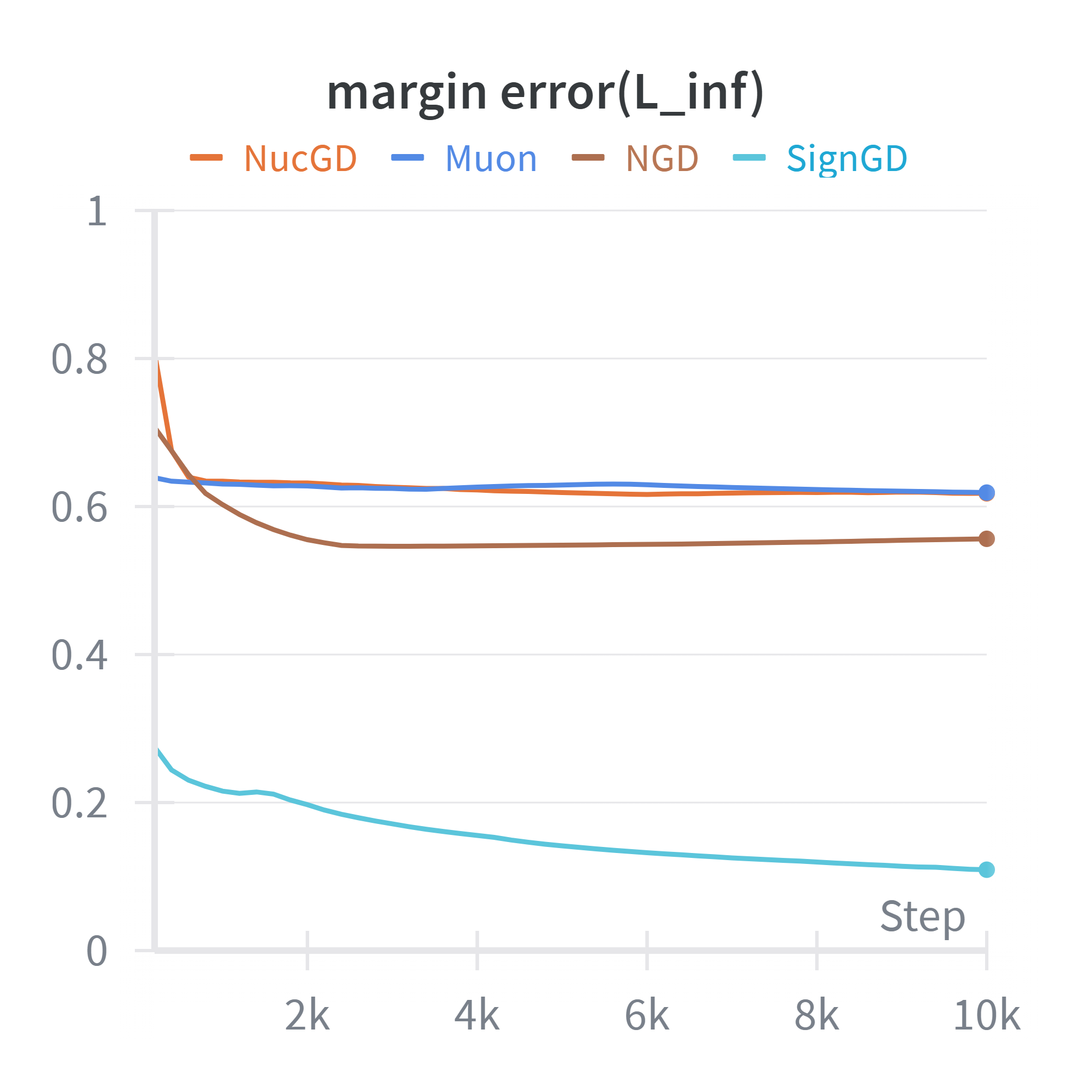}
    \centerline{}
\end{minipage}

\caption{Margin error between GD based solution and optimal solution corresponding to specific norm }
\label{fig:err}
\end{figure}
\section{Convergence Analysis of Power Iteration Direction}\label{proof}

Consider the algorithm proposed in \ref{alg:nucgd-real}, where we perform a single step of power iteration during each gradient update: $p_t = \frac{\bm{M}_t \bm{M}_t^T p_{t-1}}{\|\bm{M}_t \bm{M}_t^T p_{t-1}\|}$. Intuitively, the gap between the \emph{exact} principal left singular vector $u_t$ and the \emph{empirical} principal left singular vector $p_t$ will vanish when the algorithm converges. Here we establish a theoretical guarantee for this intuition.

We use the tangent of the angle (denoted by $\angle(\cdot, \cdot)$) to measure the discrepancy between two unit vectors. Without loss of generality, we assume $\angle(u_t, u_{t-1}), \angle(u_t, p_t) \in [0, \frac{\pi}{2}]$.
The assumptions are stated below:

\begin{enumerate}
    \item \label{ass:drift} 
    $\lim_{t\to\infty}\tan \angle(u_{t-1}, u_t)=0$.\quad(\textbf{Vanishing Drift})

    \item \label{ass:gap} 
    The ratio $\frac{\sigma_2(\bm{M}_t)}{\sigma_1(\bm{M}_t)} \le \rho < 1$ holds uniformly for all $t$, where $\sigma_1(\bm{M}_t)$ and $\sigma_2(\bm{M}_t)$ are the first and second singular values of $\bm{M}_t$, and $\rho$ is a constant.\quad(\textbf{Uniform Spectral Gap})

    \item \label{ass:acute}
    $\liminf_{t\to \infty} \tan \angle(p_{t}, u_t) < \infty$.
\end{enumerate}

First, we prove a lemma stating the contraction property of the power iteration step:
\begin{lemma}
 Under Assumption \ref{ass:gap}, $\tan \angle(p_t, u_t) \leq \rho^2 \tan \angle(p_{t-1}, u_t), \quad \forall t \geq 1.$
\end{lemma}
\begin{proof}
Let $\sigma_1 \ge \sigma_2 \ge \dots \ge \sigma_{\min\{k, d\}} \ge 0$ be the singular values of $\bm{M}_t$. Then the eigenvalues of $\bm{N}_t = \bm{M}_t \bm{M}_t^T$ are $\lambda_i = \sigma_i^2$ for $i\leq \min\{k, d\}$ or $\lambda_i=0$ otherwise. We decompose the vector $p_{t-1}$ in the basis of the eigenvectors $\{u_t^{(i)}\}_{i=1}^k$ of $\bm{N}_t$:
\[
p_{t-1} = c_1 u_t^{(1)} + \sum_{i=2}^k c_i u_t^{(i)},
\]
where $u_{t}^{(i)}$ are eigenvectors corresponding to $\lambda_i$ for $i \ge 2$ and $u_{t}^{(1)}=u_t$. When $c_1=0$, the inequality holds trivially (assuming the angle is $\pi/2$, or treating tangent as infinity). For $c_1 \neq 0$, we have:
\[
\tan \angle(p_{t-1}, u_t) = \frac{\sqrt{\sum_{i=2}^k c_i^2}}{|c_1|}.
\]
Thus $\bm{N}_t p_{t-1} = c_1 \lambda_1 u_t^{(1)} + \sum_{i=2}^k c_i \lambda_i u_t^{(i)}$. The tangent of the new angle is:
\[
\tan \angle(p_t, u_t) = \frac{\sqrt{\sum_{i=2}^k c_i^2 \lambda_i^2}}{|c_1| \lambda_1}.
\]
Applying $\sqrt{\sum_{i=2}^k c_i^2 \lambda_i^2} \le \lambda_2 \sqrt{\sum_{i=2}^k c_i^2}$, we obtain:
\[
\tan \angle(p_t, u_t) \le \frac{\lambda_2}{\lambda_1} \frac{\sqrt{\sum_{i=2}^k c_i^2}}{|c_1|} = \frac{\sigma_2^2}{\sigma_1^2} \tan \angle(p_{t-1}, u_t) \le \rho^2 \tan \angle(p_{t-1}, u_t).
\]
\end{proof}

\begin{theorem}
        Under Assumptions \ref{ass:drift}--\ref{ass:acute}, we have $\lim_{t \to \infty} \langle p_t, u_t \rangle = 1.$
\end{theorem}
\begin{proof}
Denote $\delta_t = \tan\angle(u_t, u_{t-1})$ and $e_t = \tan\angle(p_t, u_t)$.

Choose a constant $N>0$ such that for all $t\geq N$, $\delta_t \leq \min\{\frac{K(1-\rho^2)}{K^2+\rho^2}, \frac{1-\rho^2}{1+\rho^2}\}$ (by Assumption \ref{ass:drift}).

Take a finite upper bound $K >\liminf_{t\to\infty}e_t$ and find $T> N$ with $e_T < K$ (by Assumption \ref{ass:acute}). 

Consider the angle between the previous estimate $p_{T-1}$ and the current true eigenvector $u_T$. By the triangle inequality for angles:
\begin{equation}\label{angle}
    \angle(p_{T-1}, u_T) \le \angle(p_{T-1}, u_{T-1}) + \angle(u_{T-1}, u_T).
\end{equation}
Since $\delta_t e_t\leq\frac{K^2(1-\rho^2)}{K^2+\rho^2}<1$, implying the sum of angles is less than $\pi/2$, we can apply the tangent function to \eqref{angle} and combine it with the Lemma:
\begin{equation}
\label{eq:angle_drift}
    e_T\leq\rho^2\tan \angle(p_{T-1}, u_T)\leq \rho^2\frac{\delta_{T}+e_{T-1}}{1-\delta_{T}e_{T-1}}< \rho^2\frac{\delta_T+K}{1-K\delta_T }\leq K.
\end{equation}

By induction, we can prove $e_t<K$ for all $t\geq T$. Consequently, the following contraction property holds:
\begin{equation}\label{contraction}
   e_t\leq \rho^2\frac{e_{t-1}+\delta_{t}}{1-\delta_{t}e_{t-1}}\leq \frac{\rho^2}{1-K\delta_t }(e_{t-1}+\delta_t)\leq \rho(e_{t-1}+\delta_t) \quad \text{for all } t \geq T.
\end{equation} 

Applying the inequality recursively and using Assumption \ref{ass:drift}, we get:
$$e_{T+k}\leq \rho^{k}e_{T}+\sum_{i=0}^{k-1}\rho^{k-i}\delta_{T+i+1}\leq \rho^{k}e_T + \frac{\rho}{1-\rho}\sup_{t>T}\delta_{t}\to 0 \quad\text{ as } k\to \infty.$$
This concludes the proof.

\begin{remark}
    Assumptions \ref{ass:drift} and \ref{ass:gap} are practical in training, and their geometric insight is valid. Assumption \ref{ass:acute} essentially requires that the estimate does not become orthogonal to the true vector infinitely often. In practice, one might employ a \textbf{Random Restart} strategy: restart $p_{t+1}\sim \mathcal{N}(0, I_k)$ if the norm of the projected vector collapses, which helps avoid trivial solutions and satisfy Assumption \ref{ass:acute} almost surely.
\end{remark}
\end{proof}

\end{document}